\documentclass[letterpaper, 10 pt, journal, twoside]{IEEEtran}
\IEEEoverridecommandlockouts
\usepackage{cite}
\usepackage{amsmath,amssymb,amsfonts, mathtools, amsthm}
\usepackage{algorithmic}
\usepackage{graphicx}
\graphicspath{ {./figures/} }
\usepackage{textcomp}
\usepackage{xcolor}
\usepackage{bbm}
\usepackage{comment}
\usepackage{paralist}
\usepackage[caption=false]{subfig}
\usepackage{scalerel}    
\usepackage{bm}
\usepackage{amsthm}
\theoremstyle{plain}

\DeclarePairedDelimiter\abs{\lvert}{\rvert}%

\usepackage[linesnumbered,ruled]{algorithm2e}

\theoremstyle{remark}
\newtheorem{remark}{Remark}
\theoremstyle{definition}
\newtheorem{definition}{Definition}
\theoremstyle{plain}
\newtheorem{theorem}{Theorem}

\newtheorem{lemma}{Lemma}
    
\newtheorem{subproblem}{Subproblem}  

\usepackage{dirtytalk}
\usepackage[hidelinks]{hyperref}
\let\oldnl\nl% Store \nl in \oldnl
\newcommand{\nonl}{\renewcommand{\nl}{\let\nl\oldnl}}% Remove line number for one line

\DeclareMathOperator*{\argmax}{argmax} % thin space, limits underneath in displays 
\DeclarePairedDelimiterX{\norm}[1]{\lVert}{\rVert}{#1}
\DeclareMathOperator*{\argmaxgreedy}{argmax\_greedy} % thin space, limits underneath in displays 

\usepackage{xparse}
\usepackage{hyperref} % recommended

\newcounter{problem}
\renewcommand{\theproblem}{\arabic{problem}}

\makeatletter
\NewDocumentEnvironment{problem}{o}
 {%
   \refstepcounter{problem}%
   \IfValueTF{#1}
     {\edef\problem@display{#1}}%
     {\edef\problem@display{\theproblem}}%
   \def\@currentlabel{\problem@display}%
   \noindent\textbf{Problem \problem@display.}\quad
 }
 {%
   \par\vspace{1em}%
 }
\makeatother

\begin{document}
\title{Human-in-the-Loop Multi-Robot Information Gathering with Inverse Submodular Maximization}

\author{
\IEEEauthorblockN{Guangyao Shi\textsuperscript{$\dagger$}}, 
\IEEEauthorblockN{Shipeng Liu\textsuperscript{$ \ddagger$}}, 
\IEEEauthorblockN{Ellen Novoseller\textsuperscript{*}},
\IEEEauthorblockN{Feifei Qian\textsuperscript{$ \ddagger$}},
\IEEEauthorblockN{Gaurav S. Sukhatme\textsuperscript{$\dagger$}}

\thanks{
\textsuperscript{$\dagger$} Department of Computer Science, University of Southern California, Los Angeles, CA 90089, USA. Email: \texttt{\{shig, gaurav\}@usc.edu}}
\thanks{\textsuperscript{$\ddagger$}Department of Electrical and Computer Engineering, University of Southern California, Los Angeles, CA 90089, USA. Email: \texttt{\{shipengl, feifeiqi\}@usc.edu}}
\thanks{\textsuperscript{*}U.S. Army DEVCOM Army Research Laboratory, Aberdeen Proving Ground, MD 21005, USA. \\ Email: \texttt{ellen.r.novoseller.civ@army.mil}}
%\thanks{*Author and funding information have been removed for blind review. Relevant details will be included upon acceptance.}
}

%\markboth{To appear in IEEE Transaction on Robotics}%
\markboth{IEEE Transaction}%
{Author \MakeLowercase{\textit{et al.}}: Human-in-the-Loop Multi-Robot Information Gathering with Inverse Submodular Maximization}
\IEEEpubid{}

\maketitle

\begin{abstract}
 We consider a new type of inverse combinatorial optimization, Inverse Submodular Maximization (ISM), for its application in human-in-the-loop multi-robot information gathering. 
 Forward combinatorial optimization - solving a combinatorial problem given the reward (cost)-related parameters - is widely used in multi-robot coordination. In the standard pipeline, domain experts design the reward (cost)-related parameters offline. These parameters are utilized for coordinating robots online. What if non-expert human supervisors desire to change these parameters during task execution to adapt to some new requirements? We are interested in the case where human supervisors can suggest what path primitives to take, and the robots need to change the internal decision-making parameters accordingly. We study such problems from the perspective of inverse combinatorial optimization, i.e., the process of finding parameters that give certain solutions to the problem. Specifically, we propose a new formulation for ISM for a family of multi-robot information gathering scenarios, in which we aim to find a new set of parameters that minimally deviates from the current parameters while causing a greedy algorithm to output path primitives that are the same as those desired by the human supervisors. We show that for the case with a single suggestion, such problems can be formulated as a Mixed Integer Quadratic Program (MIQP), which is intractable for existing solvers when the problem size is large. We propose a new Branch $\&$ Bound algorithm to solve such problems. For the case with multiple suggestions from several human supervisors, the problem can be cast as a multi-objective optimization and can be solved using Pareto Monte Carlo Tree Search. In numerical simulations, we demonstrate how to use ISM in multi-robot scientific data collection and event detection-driven coverage control. We show that the proposed formulations and algorithms can help robots effectively adapt to human suggestions. 
\end{abstract}


\section{Introduction}

 Multi-robot teams are increasingly being employed in scientific information-gathering applications~\cite{schlotfeldt2018anytime, dutta2021multi, xu2022interactive, tokekar2016sensor}, where mobile robots act as intelligent sensing agents that autonomously navigate in natural environments to collect spatio-temporal data. Such robotic teams can greatly reduce the workload and cognitive burden of scientists. To fully leverage the sensing and sampling potential of such robotic systems, an effective and principled decision-making framework is essential. Traditionally, such multi-robot decision-making problems have been formulated as combinatorial optimization tasks, which are typically NP-hard~\cite{parker2012decision}. In the existing literature, it is generally assumed that domain experts will carefully design the optimization objective offline. For example, as shown in Fig. \ref{fig:illustration_coverage}, in a multi-robot event detection application, experts will use the event data observed in the past and their preferences about events (e.g., how important/urgent each event is) to design an objective. The design contains several parameters related to the prior knowledge of the events and user preferences. Subsequently, robots are deployed for event detection by solving the optimization problem online and executing the plan autonomously. Many research efforts are devoted to how to (near) optimally solve the problem in a time-efficient manner. 

\begin{figure}
\centerline{\includegraphics[scale=0.6]{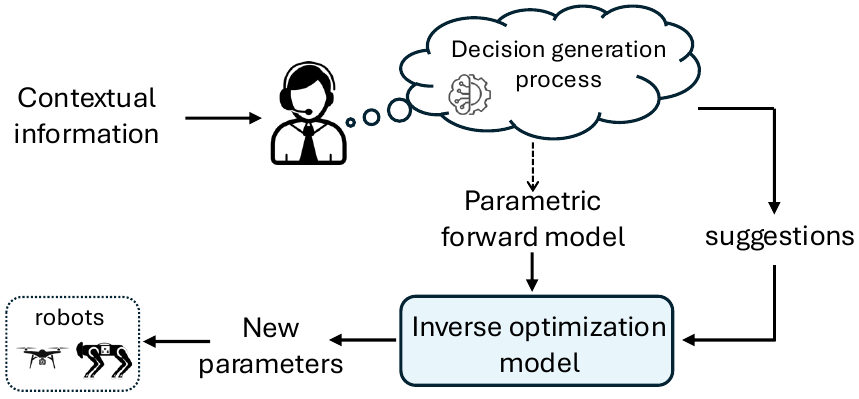}}
\caption{Schematic illustration of the proposed inverse submodular maximization framework in the context of multi-robot coordination. The inverse optimization module will take the forward parametric model and the human suggestions as input and output the updated parameters for forward models.
}
\label{fig:overview}
\end{figure}

However, when robot teams are deployed in the field, rather than purely relying on robotic autonomy, a human supervisor will usually stay in the loop to watch over the robotic team. They may receive additional information from external sources or extract additional information based on their experience and the current observations, and occasionally suggest possibly better and different path primitives to take compared to those obtained from solving the original decision-making problems designed offline. In other words, there is a mismatch between the decision-making model used by the human supervisor and the one used by robot teams. Such differences reflect either an oversight in the design phase, the arrival of new information/instructions, or a change in human preferences. In fact, there are already research showing that as human scientists collect samples from environments, they may update their sampling strategies and priorities in field tests compared to what they thought before testing to facilitate better scientific discovery~\cite{wilson2021spatially, liu2024understanding, liu2024modelling}.  When facing such differences, it is undesirable to stop the team and redesign the decision-making process. Moreover, the human supervisors in the loop may not be optimization experts, and they may not know how to recast the optimization problem to accommodate their new insights. By contrast, \textit{we want the robot team to have the capability to modify the decision-making problem in a \textbf{minimal} way to accommodate new supervisory suggestions}. The reason for the minimal change is that the existing decision-making problem already embodies substantial expertise and historical data, and human preferences, when facing trade-offs, tend to evolve through small, cumulative adjustments rather than sudden, dramatic changes \cite{liu2024modelling}. We should not discard the previous parameters entirely in response to a few new suggestions.

\begin{figure}
\centerline{\includegraphics[scale=0.4]{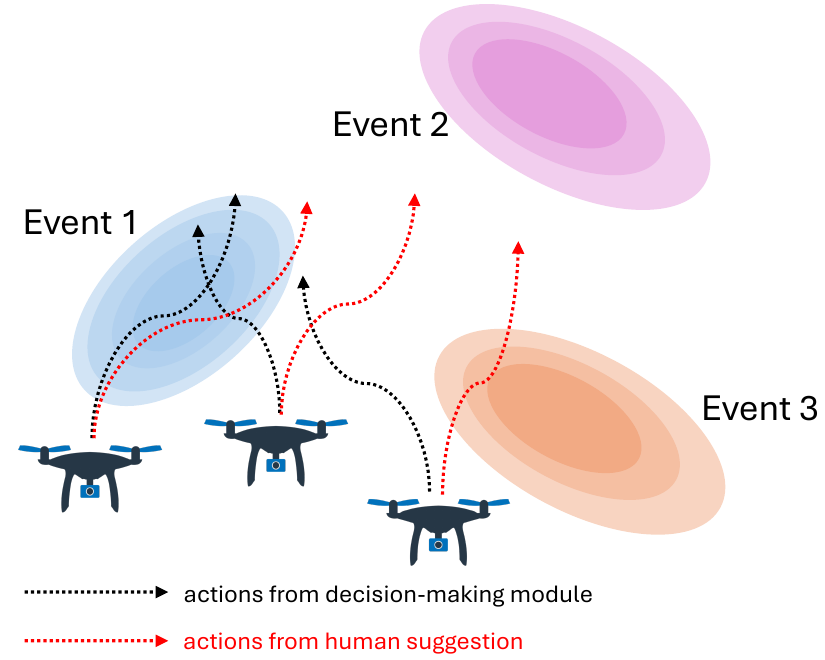}}
\caption{A motivating example of inverse submodular maximization. A team of robots is deployed to detect multiple events, each of which is associated with a priority, and the task is cast as a submodular maximization problem. Black dotted lines are path primitives derived from task optimization, and red dotted lines are path primitives suggested by humans. The team needs to minimally adjust the submodular objective to account for human suggestions. 
}
\label{fig:illustration_coverage}
\end{figure}

In this paper, we explore using inverse combinatorial optimization to equip robots with the capabilities to accommodate human suggestions during information gathering tasks. Specifically, we focus on a family of decision-making problems for multi-robot information gathering, submodular maximization. In multi-robot information gathering, many objectives (e.g., mutual information \cite{krause2008near, hollinger2014sampling, shi2023robust}, area coverage~\cite{poduri2004constrained},
target search and tracking \cite{ xu2025communication, xu2023online, Xu2023BanditSubmodularMaximization, shi2021communication, li2025multi, xu2023online}, detection probability \cite{tseng2017near, wu2022adaptive} etc.) have a diminishing returns
property, i.e., submodularity. Intuitively, submodularity formalizes the notion that adding a robot to a larger multi-robot team will yield a smaller marginal gain
in the objective than adding the same robot to a smaller team.  We are interested in considering decision-making problems that involve maximizing such parameterized submodular objectives, which is NP-hard but can be solved with an $(1-\frac{1}{e})$-approximation by a greedy algorithm \cite{nemhauser1978analysis}.
For such problems with a parameterized submodular objective, e.g, \cite{wilder2019melding, shi2023decision}, if the parameters are known, we can use a greedy approximation algorithm and its variants to solve the problem and obtain a nearly optimal solution. We call such a process Forward Submodular Maximization (FSM). By contrast, if we have a solution suggested as the output of an approximation algorithm, we want to find \textit{minimally} modified parameters in the objective compared to the original parameters such that when we use those modified parameters to solve the FSM problem, the resulting solutions match the suggested solutions. We call this process Inverse Submodular Maximization (ISM), as illustrated in Fig. \ref{fig:overview}. ISM problems fall into the more general class of inverse optimization problems, which are widely used to decode human decision-making processes \cite{terekhov2010analytical, tsirakos1997inverse, sandholtz2023inverse} for robotic applications. 

Here, we study such ISM problems for multi-robot information gathering, formulated as FSM problems, to adapt to changes in human preferences, which are defined as different trade-off strategies between multiple objectives. Specifically, we formulate two types of ISM problems motivated by common scenarios in human–multi-robot collaborative scientific data collection.
In the first type, a human supervisor provides a single suggestion regarding the path primitives to be taken. The goal is to adjust the parameters minimally so that the approximation algorithm produces a solution identical to the human suggestion when using the updated parameters.
In the second type, the human supervisors provide multiple suggestions, {each paired by a confidence score}. The objective is to identify new parameters for the submodular objective function that minimally deviate from the originally designed parameters while ensuring that the resulting solution is “close” to the human suggestions. 
To the best of the authors' knowledge, this is the first comprehensive study on formulating ISM problems for adapting human preference in multi-robot information gathering. We note that even in the optimization literature, the ISM problem is not well explored, and there is no existing published canonical formulation of ISM. In addition to the ISM formulations, we introduce a series of new algorithms to solve the ISM problems.

The main contributions of this paper are:
\begin{itemize}
    \item We formulate a new family of inverse optimization problems, ISM, for a family of human-multi-robot collaborative data collection applications.
    \item We propose several efficient algorithms to solve the proposed ISM problems and validate their effectiveness through extensive simulations.
    \item We demonstrate the applicability of the proposed formulations to multi-robot information-gathering problems through two representative case studies: multi-robot event detection and multi-robot geological data collection.
\end{itemize}

%%%%%%%%%%%%%%%%%%%%%%%%%%%%%%%%%%%%%%%%%%%%%%%%%
\section{Related work} \label{sec:relwork}
\textbf{Inverse Optimization:}
Inverse optimization is gaining increasing attention over the past years and is studied for a diverse range of applications, including vehicle routing \cite{scroccaro2023inverse,  chen2021inverse}, power system \cite{tan2023data, hu2018inverse}, and robotics \cite{mombaur2010human, jin2021inverse, finn2016guided}. Based on the nature of the forward optimization problem, we can roughly classify inverse problems into two categories: continuous case and discrete case. Most existing work related to robotics fits into the continuous category.  Researchers use Inverse Optimal Control (IOC) or Inverse Reinforcement Learning (IRL) to refer to inverse optimization. IOC/IRL is usually studied to learn the human behaviors for one robot, for example, human walking for humanoid locomotion and human grasping skills for robot arm manipulation. Extending frameworks of IOC/IRL from one robot to multiple robots is not a trivial task, and research along this line is still less mature \cite{vsovsic2016inverse, yu2019multi, wang2018competitive}.   
By contrast, the discrete case is less considered in the robotics literature. The discrete case corresponds to combinatorial optimization, which is widely used in multi-robot coordination, in the forward problem. In the optimization community, such a discrete inverse problem is referred to as Inverse Combinatorial Optimization (ICO) \cite{heuberger2004inverse}, among which a subbranch called the Inverse Integer Optimization (IIO) is widely studied \cite{schaefer2009inverse}. However, IIO can not deal with a submodular objective. The ISM formulations proposed in this paper can be viewed as another sub-branch of ICO, and this is the first paper to study such problems in depth in the context of multi-robot coordination.  Existing work \cite{shi2024inverse} considers only a restricted subproblem (single suggestion) of the broader problem studied in this paper and does not include applications to scientific data collection.

\textbf{Human Preference Learning:}
ISM is also related to the research on human preference learning. In these works \cite{biyik2018batch, wilde2020improving, biyik2019asking, wilde2020active},  humans are usually given two solutions iteratively and need to compare these two solutions for each iteration. Most research focuses on how to efficiently generate queries for humans to learn the preference parameters. By contrast, in ISM, humans are presented with a ground set, and they will select a subset that they think the \textit{approximation algorithm} should output. The research focus is on how to minimally modify the parameters in the objective to match the human suggestions. There is no iterative process in ISM because ISM targets online information gathering scenarios in which the tasks will not be repeated many times. By contrast, existing human preference learning literature focuses on a task (e.g., pick-and-place type manipulation tasks) that will be repeated many times, and the goal is to make the robot replicate humans' behavior in future repetitive trials. 
%\todo{cite some arguments from Shipeng's paper; explain why not multiple rounds: cannot stop and time sensitive stuff; preference learning is to learn some repetitive task.}

\section{Preliminaries}\label{sec:preliminaries}
We use calligraphic fonts to denote sets (e.g. $\mathcal{A}$). Given a set $\mathcal{A}$, $2^{\mathcal{A}}$ denotes the power set of $\mathcal{A}$ and $|\mathcal{A}|$ denotes the cardinality of $\mathcal{A}$. Given another set $\mathcal{B}$, the set $\mathcal{A} \setminus \mathcal{B}$ denotes the set of elements in $\mathcal{A}$ but not in $\mathcal{B}$. 
%Given a set $\mathcal{V}$, a set function $f: 2^{\mathcal{V}} \mapsto \mathbb{R}$, and an element $s \in \mathcal{V}$, $f(s)$ is a shorthand that denotes $f(\{s\})$. 
%We use $f_{\mathcal{A}}(\mathcal{B})$ to denote $f(\mathcal{A} \cup \mathcal{B})-f(\mathcal{A})$. The submodularity of a set function is defined below. 
We will use $\Delta f (s \mid \mathcal{S})$ to denote the marginal gain of adding an element $s$ to a set  $\mathcal{S}$, i.e.,
$\Delta f (s \mid \mathcal{S}) = f(\{s\} \cup \mathcal{S}) - f( \mathcal{S})$, where $f$ is an objective function that we wish to maximize. We will denote $\norm{\cdot}$ as the $\ell^2$-norm. A set is an ordered set if it can preserve the order in which we insert the elements. For an ordered set $\mathcal{S}$, we use $\mathcal{S}[i]$ and $\mathcal{S}[1:i]$ to refer to its $i$-th element and its first $i$ elements, respectively. For an unordered set, we call each permutation of the set as an \textit{ordering} of the set.
We use a capitalized bold letter such as $\mathbf{X}$ to denote a random vector, use a bold letter $\mathbf{x}$ to denote a realization of $\mathbf{X}$, use $X$ to denote locations, and use $x/\bm{x}$ to denote a general variable (scalar $/$ vector).

We now define two useful properties of set functions.
\begin{definition}[Normalized Monotonicity]
For a set $\mathcal{V}$, a function $f: 2^{\mathcal{V}} \mapsto \mathbb{R}$ is called normalized and monotonically non-decreasing if and only if (1) for any $\mathcal{A} \subseteq \mathcal{A}^{\prime} \subseteq \mathcal{V}, f(\mathcal{A}) \leq f(\mathcal{A}^{\prime})$ and (2) $f(\mathcal{A})=0$ if and only if $\mathcal{A}=\emptyset$.
\end{definition}
As shorthand, we refer to a normalized, monotonically non-decreasing function as simply a monotone function.
\begin{definition}[Submodularity]
For a finite ground set $\mathcal{V}$, a function $f: 2^{\mathcal{V}} \mapsto \mathbb{R}$ is submodular if the following condition is satisfied: for every $\mathcal{A}, \mathcal{B} \subseteq \mathcal{V}$ with $\mathcal{A} \subseteq \mathcal{B}$ and every $v \in \mathcal{V} \setminus \mathcal{B}$ we have that 
\begin{equation}\label{eq:submodularity_definition}
    f(\mathcal{A} \cup \{v\}) -f(\mathcal{A}) \geq f(\mathcal{B} \cup \{v\}) -f(\mathcal{B}).
\end{equation}
%for any sets $\mathcal{A} \subseteq \mathcal{V}$ and $\mathcal{A}^{\prime} \subseteq \mathcal{V}$, we have $f(\mathcal{A})+f(\mathcal{A}^{\prime}) \geq f(\mathcal{A} \cup \mathcal{A}^{\prime})+f(\mathcal{A} \cap \mathcal{A}^{\prime})$.
\end{definition}
Similarly, if the LHS of Eq. \eqref{eq:submodularity_definition} is always no greater than the RHS of Eq. \eqref{eq:submodularity_definition}, the function $f$ is supermodular.

\subsection{Forward Submodular Maximization}
 Many objectives in multi-robot information gathering applications exhibit submodularity~\cite{krause2008near, hollinger2014sampling, shi2023robust, poduri2004constrained,  xu2023online, Xu2023BanditSubmodularMaximization, shi2021communication, li2025multi}.  
We consider a problem setup with $n_r$ robots, and each robot has a set of path primitives. 
At the decision-making moment, robots need to jointly maximize monotone submodular functions by selecting a primitive to follow to obtain as much information as possible. Two specific examples will be discussed in Sec. \ref{sec:case_studies}. Mathematically, we want to solve the following optimization problem:
\begin{equation}\label{eq:SM}
    \begin{aligned}
        \max_{\mathcal{S} \subseteq \mathcal{P}} &~f(\mathcal{S}, \bm{e}, \bm{\theta}) \\
        \text{s.t.} ~& ~\abs{\mathcal{S} \cap \mathcal{P}_i} \leq \kappa_i, ~i=\{1, \ldots, n_r\}, \kappa_i \in \mathbb{Z}_{\geq 1}. 
    \end{aligned}
\end{equation}
where $f$ is a submodular objective and is characterized by two parameters: $\bm{e} \in \mathcal{E}$ denotes the environmental state variables and $\bm{\theta} \in \Theta$ is a vector of weights denoting the human preference on how to trade off between objectives; $\mathcal{P}_i$ denotes the set of path primitives for robot $i$ and $\mathcal{P}=\cup_{i=1}^{n_r} \mathcal{P}_i, ~ \mathcal{P}_i \cap \mathcal{P}_j =\emptyset ~ \forall i \neq j$ is the ground set of path primitives for all robots; $\mathcal{S}$ is the selected path primitives for robots; $\kappa_i$ denotes the number of path primitives we can select for one robot and is usually set to be one; $n_r$ denotes the number of robots.

In the context of information gathering, $\bm{e}$ can be viewed as the output of the perception (e.g., mapping and localization) module. It is a feature vector representing what the environment is like. Humans' suggestions will not affect $\bm{e}$. Our focus for ISM is human-preferences related parameters $\bm{\theta}$. Two representative examples of the submodular objectives will be discussed in Sec. \ref{sec:case_studies}.

As an initial trial to study the inverse version of submodular maximization, we will confine our discussion to the case where $f(\mathcal{S}, \bm{e}, \bm{\theta})$ is a linear combination of submodular basis functions \cite{wilde2021learning, tseng2017near}, i.e.,
\begin{equation}\label{eq:linear_submodular}
    f(\mathcal{S}, \bm{e}, \bm{\theta}) =  \bm{\theta}^{T}g(\mathcal{S},  \bm{e}),
\end{equation}
where $g( \mathcal{S}, \bm{e})=[g_1(\mathcal{S}, \bm{e}), ..., g_{n_o}(\mathcal{S}, \bm{e})]^T$ is a set of monotone submodular functions and $n_o$ denotes the number of objectives. The objective function $f$ of such form is commonly used in multi-objective decision-making problems (e.g., information gathering) where we need to make trade-offs by weighting different objectives using $\bm{\theta}$. We leave the more general case for future work. For brevity, we will ignore the $\bm{e}$ in the objective in the rest of the paper.

For a given $\bm{\theta}$, the forward problem in Eq. \eqref{eq:SM} can be solved nearly optimally using the greedy algorithm \cite{nemhauser1978analysis} as shown in Algorithm \ref{algorithm:greedy} in which $\mathcal{I} = \{ \mathcal{S} \subseteq \mathcal{P} \mid \abs{\mathcal{S} \cap \mathcal{P}_i} \leq \kappa_i, ~i=\{1, \ldots, n_r\}, \kappa_i \in \mathbb{Z}_{\geq 1}\}$.
\begin{algorithm}[ht]\label{algorithm:greedy}
    \caption{Greedy Algorithm}
    \SetKwInOut{Input}{Input}
    \SetKwInOut{Output}{Output}
    %\underline{function Greedy}($f,\{\mathcal{T}_i\}_{i=1}^{N}, \mathcal{P}$) \\
    \Input{
    \begin{itemize}
        \item A monotone submodular function $f$
        \item A partition matroid $\mathcal{M}=(\mathcal{P}, ~\mathcal{I})$
        %\item Operator $\mathcal{P}$ that  can  extract  the  end  points of  trajectories
    \end{itemize}
    }
    \Output{
    A subset $\mathcal{S} \in \mathcal{I}$ of the ground set $\mathcal{P}$
    }
    $\mathcal{S} \gets \emptyset$ \\
    \While{$\abs{\mathcal{S}} < n_r$}{
    \nonl \# find the element with the largest marginal gain \\
    $s = \argmax_{s \in \mathcal{P}, \{s\} \cup \mathcal{S} \in \mathcal{I}}~\Delta f(s \mid \mathcal{S})$ \\
    %$\Tilde{\mathcal{T}} \gets \Tilde{\mathcal{T}} \setminus s_i$ \\
    $\mathcal{S} \gets \mathcal{S} \cup \{s\}$
    }
    return $\mathcal{S}$
\end{algorithm}

\subsection{Gaussian Process Regression}
\begin{definition}[Gaussian Process]
Let $\mathcal{X}$ be an input space (e.g.\ $\mathbb{R}^d$), and let
$f : \mathcal{X} \to \mathbb{R}$ be a random function, i.e., a
collection of real-valued random variables indexed by $x \in \mathcal{X}$:
$
\{ f(x) : x \in \mathcal{X} \}.
$
We say that $f$ is a \emph{Gaussian process ($\mathcal{GP}$)} if for any finite
set of points $x_1, \dots, x_n \in \mathcal{X}$, the random vector
\(
\bigl[f(x_1), \dots, f(x_n)\bigr]^{T} \in \mathbb{R}^n
\)
has a multivariate Gaussian distribution.

A Gaussian process is fully specified by a \emph{mean function}
    $
    m(x) = \mathbb{E}[f(x)], \quad x \in \mathcal{X},
    $
 and a \emph{covariance (kernel) function}
$
k(x, x') = \operatorname{cov}\bigl(f(x), f(x')\bigr),
    \quad x, x' \in \mathcal{X}.
$
We write this compactly as:
\[
f \sim \mathcal{GP}\bigl(m(\cdot), k(\cdot, \cdot)\bigr).
\]

For any finite collection $x_1, \dots, x_n \in \mathcal{X}$, this implies
\[
\mathbf{f} =
\begin{bmatrix}
f(x_1)\\
\vdots\\
f(x_n)
\end{bmatrix}
\sim
\mathcal{N}\bigl(
\boldsymbol{\mu},
\mathbf{K}
\bigr),
\]
where $\mu_i = m(x_i)$ and $K_{ij} = k(x_i, x_j)$.
\end{definition}

In many multi-robot information gathering applications, we need to coordinate robots to learn a map $f: \mathcal{X} \to \mathbb{R}$ (i.e., a function) that takes locations as input and outputs certain physical properties about the locations. Such a mapping can denote occupancy of space~\cite{corah2019distributed}, energy consumption~\cite{quann2020off}, temperature~\cite{popovic2020informative}, biological concentration~\cite{fonseca2023adaptive}, etc.  In practice, the structure of the underlying function is usually unknown, and its analytical evaluation may be challenging. In such cases, Gaussian process regression is often used as an alternative approach in describing, actively learning, and optimizing unknown functions~\cite{schulz2018tutorial, williams2006gaussian}. 

\begin{definition}[Gaussian Process Regression]
Let $\mathcal{X}$ be an input space and let
$f : \mathcal{X} \to \mathbb{R}$ be a stochastic process with a
Gaussian process prior
$
f \sim \mathcal{GP}\bigl(m(\cdot), k(\cdot, \cdot)\bigr),
$
where $m : \mathcal{X} \to \mathbb{R}$ is the mean function and
$k : \mathcal{X} \times \mathcal{X} \to \mathbb{R}$ is the covariance (kernel) function.

Given training inputs $X = (x_1, \dots, x_n)^\top$ and corresponding
observations $y = (y_1, \dots, y_n)^\top$, we assume that each sample is obtained using an observation model
\(
y_i = f(x_i) + \varepsilon_i, \quad \varepsilon_i \sim \mathcal{N}(0, \sigma_i^2)
\)
independently for $i = 1, \dots, n$, where $\sigma_i^2 > 0$ is the noise variance.

For a set of test inputs $X_* = (x_1^*, \dots, x_m^*)^\top$, denote 
$f_* = (f(x_1^*), \dots, f(x_m^*))^\top$ as the corresponding (latent) function values.
Then, under the GP prior and the Gaussian noise model, the joint distribution of
$y$ and $f_*$ is multivariate normal, and the posterior (predictive) distribution
of $f_*$ given $(X, y, X_*)$ is
\(
p(f_* \mid X, y, X_*) = \mathcal{N}\bigl(\mu_*, \Sigma_*\bigr),
\)
where
\begin{align}
\resizebox{0.9\linewidth}{!}{$
\mu_* = m(X_*) + K(X_*,X)\bigl[K(X,X) + \sigma^2 I\bigr]^{-1}\bigl(y - m(X)\bigr)
$}, \label{eq:GPR:mean}\\
\resizebox{0.9\linewidth}{!}{$
\Sigma_* = K(X_*,X_*) - K(X_*,X)\bigl[K(X,X) + \sigma^2 I\bigr]^{-1}K(X,X_*)
$}. \label{eq:GPR:variance}
\end{align}

Here $K(X,X)$ is the $n \times n$ covariance matrix with entries
$K(X,X)_{ij} = k(x_i, x_j)$, $K(X_*, X)$ is the $m \times n$ matrix with
entries $k(x_i^*, x_j)$, and $K(X_*, X_*)$ is the $m \times m$ matrix with
entries $k(x_i^*, x_j^*)$.
The resulting mapping from $(X,y)$ to the posterior distribution
$p(f_* \mid X,y,X_*)$ is called \emph{Gaussian process regression}.
\end{definition}

\section{Problem Formulation}
This section presents formal formulations of ISM. Specifically, we consider two types of ISM. In the first type, a human supervisor provides a single suggestion to the robot team for preference adaptation. A suggestion consists of a set of distinct primitives, with one primitive assigned to each robot.
In the second type of formulation, human supervisors provide multiple suggestions. Such a formulation is suitable for cases where multiple experts are involved to tune the decision-making objective, and people propose different solutions.
ISM formulations presented in this section can be used as a complementary part of the existing multi-robot coordination framework based on submodular maximization to accommodate human suggestions.

\subsection{Case I: A single human suggestion}

\begin{problem}[1] \label{prob:single_suggestion} Given a suggestion set $\hat{\mathcal{S}} \subseteq \mathcal{P}$ from a human, which includes one path primitive per robot, find a new parameter vector $\hat{\bm{\theta}}$ such that the distance between the original parameter vector $\bm{\theta}$ and the new parameter vector $\hat{\bm{\theta}}$ is minimized and the human suggestion $\hat{\mathcal{S}}$ becomes the solution for the new problem using $\hat{\bm{\theta}}$ in the forward approximation algorithm, i.e., greedy selection. Mathematically, the inverse problem can be formulated as:
    \begin{align}
    \min_{\hat{\bm{\theta}}}~ &\norm{\hat{\bm{\theta}} - \bm{\theta}} \\
    \rm{s.t.} \quad & {\mathcal{S}(\hat{\bm{\theta}})} = \argmaxgreedy_{\mathcal{S}} f(\mathcal{S}, \hat{\bm{\theta}}), \label{eq:single:new_sol}\\
    & \hat{\mathcal{S}} = \mathcal{S}(\hat{\bm{\theta}}),\label{eq:single:equal}
    %& \hat{\bm{\theta}} \in \bm{\Theta}.
    \end{align}
where ${\mathcal{S}(\hat{\bm{\theta}})}$ is the solution returned by the greedy algorithm, and the constraint \eqref{eq:single:equal} enforces the solution to be equal to the human suggestion. 
\end{problem}

\subsection{Case II: Multiple human suggestions}\label{sec:problem:multiple_suggestion}
We also consider the case where several human supervisors give multiple suggestions, and they are not sure which one is the most desirable solution. Instead, we assume that the human supervisors will give a score for each of their suggestions to reflect their confidence in the suggestion. We will introduce a novel formulation to handle such a case. Besides, the following formulation is also suitable to model the case where a human supervisor gives multiple suggestions to express the uncertainty in their preferences.

\begin{problem}[2]\label{prob:multiple_suggestion} Given $m$ suggestions $\{(\hat{\mathcal{S}}_1, w_1),  \ldots,  (\hat{\mathcal{S}}_m, w_m)\}$ from human supervisors, find a new parameter vector $\hat{\bm{\theta}}$ such that the following objectives are minimized:
\begin{itemize}
    \item The distance between the original parameter vector $\bm{\theta}$ and the new parameter vector $\hat{\bm{\theta}}$, 
    \item  The weighted sum of differences in the objective values, 
    \item The weighted sum distance between the decisions obtained using $\hat{\bm{\theta}}$ and the human suggestion.
\end{itemize}
 Mathematically, the inverse problem can be formulated as:
\begin{align}
    \min_{\hat{\bm{\theta}}}~ & \norm{\hat{\bm{\theta}} - \bm{\theta}}, \quad \sum_{i=1}^m w_i \epsilon_i^f, \quad \sum_{i=1}^m w_i \epsilon_i^s,  \\
    \rm{s.t.} \quad  {\mathcal{S}}(\hat{\bm{\theta}}) &= \argmaxgreedy_{\mathcal{S}} f(\mathcal{S}, \hat{\bm{\theta}}), \label{eq:multiple:greedy_selection}\\
     \epsilon_i^f &= \norm{f({\mathcal{S}}(\hat{\bm{\theta}}))- f(\hat{\mathcal{S}}_i)} , ~\forall i \in [m], \label{eq:multiple_suggestion_objective_value}\\
     \epsilon_i^s &= d(\mathcal{S}(\hat{\bm{\theta}}), \hat{\mathcal{S}}_i) , ~\forall i \in [m], \label{eq:multiple_suggestion_set_distance}
\end{align}
 where $\epsilon_i^f$ is an auxiliary variable representing the difference in objective values between the suggestion and the solution, and $d(\mathcal{S}_1, \mathcal{S}_2)$ denotes Hamming distance, which is defined as the number of robots whose corresponding path primitives are different, given two selected sets $\mathcal{S}_1$ and $\mathcal{S}_2$. 
 
\end{problem}

The intuition behind the above formulation is as follows. When we have multiple suggestions, we want to find a minimally modified parameter vector (corresponding to the first objective) such that when we make decisions using this new parameter vector, the resulting decision is closer to the suggestion with a higher confidence score w.r.t. some solution metrics. This is reflected in the second and third objectives: the distance function is weighted by the confidence score. As a result, if a particular suggestion $\hat{\mathcal{S}}_i$ has a high confidence score $w_i$, we would expect to get a solution where the corresponding distance metric w.r.t. to the output solution is small, and the corresponding objective values should be close.

\begin{remark}
In Problem \ref{prob:single_suggestion}, we enforce that the solution obtained using $\hat{\bm{\theta}}$ should be the same as the one suggested by the human supervisor (Eqs. \eqref{eq:single:new_sol} and \eqref{eq:single:equal}). Such hard constraints may make the resulting problem infeasible. By contrast, when there are multiple suggestions, as in Problem \ref{prob:multiple_suggestion}, we cannot enforce such hard constraints, which will obviously make the problem unsolvable. Instead, we just want to find the solution and the corresponding $\hat{\bm{\theta}}$ that is close to the human suggestions w.r.t. some metrics (as defined in Eq. \eqref{eq:multiple_suggestion_objective_value} and Eq. \eqref{eq:multiple_suggestion_set_distance}). There is always a feasible solution for such problems. If Problem \ref{prob:single_suggestion} is infeasible and a human supervisor cannot provide an alternative suggestion, we can turn that one suggestion into a special case of Problem \ref{prob:multiple_suggestion} to adapt to their preference. 
\end{remark}

\section{Case Studies}\label{sec:case_studies}
This section presents two case studies that motivate us to propose ISM for multi-robot information gathering. 

\begin{figure}[t]
  \centering
  \includegraphics[width=0.9\columnwidth]{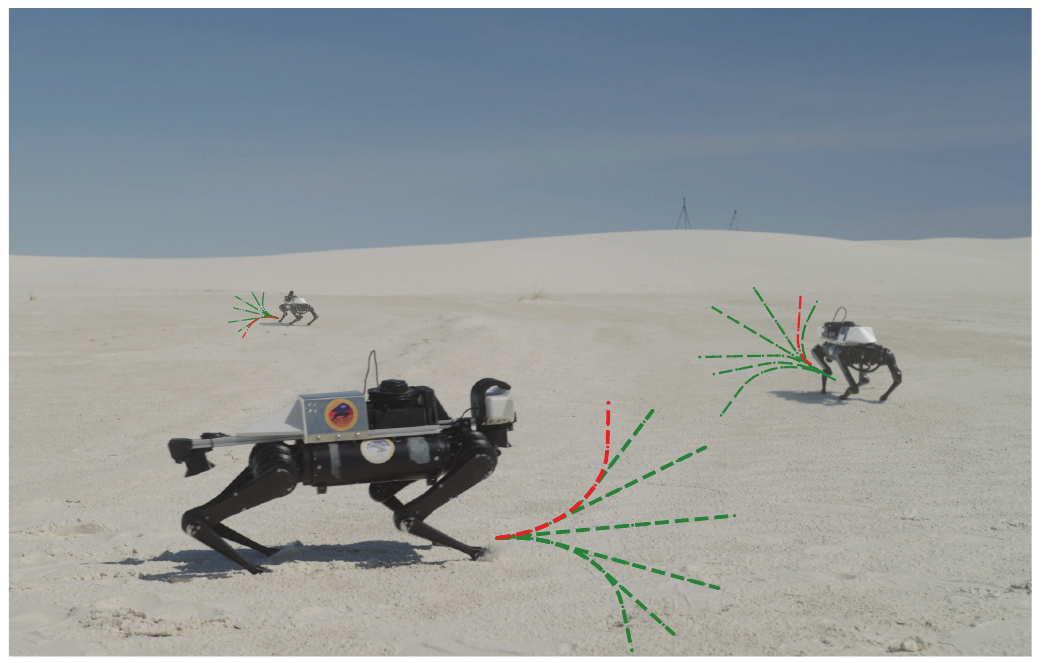}
  \caption{A motivating example of inverse submodular maximization. A team of legged robots uses onboard moisture sensing and proprioceptive feedback from leg–terrain interactions to collect soil information and validate hypotheses about subsurface moisture and strength. Dotted lines denote path primitives available to the robots.
  }
  \label{fig:illustration_data_gathering}
\end{figure}

% \begin{figure}
% \centerline{\includegraphics[scale=0.5]{figures/multi-robot_coverage.pdf}}
% \caption{A motivating example of inverse submodular maximization. 
% }
% \label{fig:illustration_data_gathering}
% \end{figure}
\subsection{Multi-Robot Scientific Data Collection}\label{sec:case_studies:scientific}
This case study on multi-robot scientific data collection is inspired by \cite{liu2024modelling, liu2024understanding, Rankin2026RobotAssistedExploration, hiatt2011accommodating, candela2017planetary, hollinger2014sampling}.  
Robotic teams are increasingly being deployed to assist human field scientists in Earth and planetary science tasks~\cite{liu2025scout}, where deformable terrain~\cite {liu2023adaptation} is common. The robots can provide geoscientists with both high quality and spatial density surface measurements through
\textit{in-situ} physical interactions~\cite{liu2025adaptive, fulcher2025effect,qian2019rapid, ruck2024downslope}. Having access to these data allows scientists to evaluate hypotheses regarding soil strength, moisture, temperature, and/or microbial activity under field conditions~\cite{rossel2011proximal}; update their conceptual models; and dynamically adjust their data collection approaches to facilitate significant scientific breakthroughs. Generally, such problems are modeled as scalarized multi-objective decision-making problems. The robotic teams need to make a trade-off between obtaining more information about physical, chemical, or biological properties of the soil and exploiting existing information to validate scientists' hypotheses (e.g., dry
soil has the lowest strength, and soil moisture increases with soil strength,
eventually reaching a plateau as soil moisture saturates~\cite{liu2024understanding}). Such trade-offs can be handled by considering human preferences. 

At a high level, we consider a team of $n_r$ robots. Each robot is capable of measuring certain soil properties. At the decision-making moment, each robot has a set of path primitives, each of which will lead the robot to take the measurements along a particular path. The decision-making problem is to select the ``best" path primitive for each robot to collectively optimize the preference-weighted objective.

\textbf{{Environment Models}:} 
% A tutorial on Gaussian process regression: Modelling, exploring, and exploiting functions
The task environment $\Omega \subset \mathbb{R}^2$ is a convex polygon embedded in 2-dimensional space. There is a set of discrete locations of interest $\mathcal{V}$, and we want to estimate field properties at these locations by taking measurement in $\Omega$.  Each location of interest is associated with several environmental properties that can be measured, e.g., soil strength and moisture. Each environmental property in the task environment can be viewed as an unknown function defined over $\Omega$, and we assume that we will use a GP with known parameters to approximately represent such a function. As a result, for any environmental property, the joint distribution at all candidate locations has a (multivariate) Gaussian joint distribution. With this GP representation, if we observe a set of sensor measurements \( \mathbf{Y}_A = \mathbf{y}_A \) corresponding to the finite subset of locations \( X_A \subset \Omega \), we can predict the value at any point \( v \in \mathcal{V} \) conditioned on these measurements, \( P(g(x_v) \mid \mathbf{Y}_A = \mathbf{y}_A) \) using Eq. \eqref{eq:GPR:mean} and Eq. \eqref{eq:GPR:variance}.

\iffalse
The distribution of \( X_y \), given these observations, is a Gaussian whose conditional mean \( \mu_{y|A} \) and variance \( \sigma_{y|A}^2 \) are given by:

\[
\begin{aligned}
\mu_{y|A} &= \mu_y + \Sigma_{yA}\Sigma_{AA}^{-1}(x_A - \mu_A), \\
\sigma_{y|A}^2 &= \mathcal{K}(y, y) - \Sigma_{yA}\Sigma_{AA}^{-1}\Sigma_{Ay},
\end{aligned}
\tag{1,2}
\]
where \( \Sigma_{yA} \) is a covariance vector with one entry for each \( u \in A \) with value \( \mathcal{K}(y, u) \), and \( \Sigma_{Ay} = \Sigma_{yA}^T \).
\fi

\textbf{{Robots and Sensing}:} 
The robots are initially located in one area in $\Omega$. At each decision-making moment, each robot has a set of path primitives, each of which consists of a fixed number of discrete sampling points. We will use $\mathcal{P}$ to denote the set of all path primitives. With slight abuse of notation, we use $p \in \mathcal{P}$ to refer to the path as
well as to the set of vertices visited on the path. The decision here is to determine which primitive to select for each robot.  Each path primitive enables the robots to measure several environmental properties at each point along the path primitive. 

\textbf{Information gathering objective:} Entropy,  \(H(\mathbf{X})\), quantifies the expected amount of information needed to determine the outcome of a random vector $\mathbf{X}$, where a higher entropy indicates greater uncertainty, while a lower entropy indicates a more predictable outcome. Conditional entropy, \(H(\mathbf{X}|\mathbf{Y})\), is the measure of the uncertainty remaining in a random vector \(\mathbf{X}\) given the value of another random vector \(\mathbf{Y}\).
When \(\mathbf{X}\) and \(\mathbf{Y}\) are jointly Gaussian, these two metrics can be computed as 
\begin{equation}\label{eq:entropy}
    H(\mathbf{X}) = - \int p(\mathbf{x}) \log p(\mathbf{x}) d\mathbf{x}=\frac{1}{2}\log ((2\pi e )^{k}\abs{\Sigma_{\mathbf{X}\mathbf{X}}}),
\end{equation}
\begin{equation}\label{eq:conditional_entropy}
    \begin{aligned}
         H(\mathbf{X} \mid \mathbf{Y}) & ~= - \int \int p(\mathbf{x}, \mathbf{y}) \log p(\mathbf{x} \mid \mathbf{y}) d\mathbf{x} d\mathbf{y} \\
         & ~= \frac{1}{2} \log ((2\pi e )^{k}\abs{\Sigma_{\mathbf{X} \mid \mathbf{Y}}}),
    \end{aligned}
\end{equation}
where $k$ is the dimension of \(\mathbf{X}\).

\begin{figure}
\centerline{\includegraphics[scale=0.85]{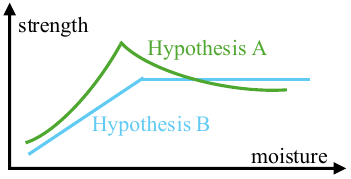}}
\caption{Two hypotheses on the relationship between soil moisture and strength.
}
\label{fig:hypothesis}
\end{figure}

In the context of scientific data collection, GPs are often used to model the environment and guide the sampling process. A typical problem can be abstracted as follows: we need to choose a subset of locations, $X_\mathcal{A}$, from a candidate set of locations, to be used for taking samples %from a finite set of candidate locations 
to maximize the information obtained from the environment. For such problems, a popular choice for the objective defined using GPs is mutual information~\cite{krause2008near, best2019dec, chen2019pareto, hollinger2014sampling}. Given two sets of random variables $\mathbf{X}_\mathcal{A}$ and $\mathbf{X}_\mathcal{B}$ corresponding to function values at locations sets $X_\mathcal{A}$ and $X_\mathcal{B}$, respectively, the mutual information can be computed as:
\begin{equation}
I\bigl( \mathbf{X}_{\mathcal{A}}; \mathbf{X}_{\mathcal{B}}\bigr)  =H\!\bigl(\mathbf{X}_{\mathcal{A}}\bigr) \;-\; H\!\bigl(\mathbf{X}_{\mathcal{A}}\mid \mathbf{X}_{\mathcal{B}}\bigr), 
\end{equation}
where  $H\!\bigl(\mathbf{X}_{\mathcal{A}}\bigr)$ and $H\!\bigl(\mathbf{X}_{\mathcal{A}}\mid \mathbf{X}_{\mathcal{B}}\bigr)$ are defined in Eq. \eqref{eq:entropy} and Eq. \eqref{eq:conditional_entropy}. 

Each environmental property $i$ is associated with a mutual information objective, which we want to maximize:

\begin{equation}\label{eq:objective:mutual_information}
    g_{i}(\mathcal{S}) = I(\mathbf{X}_{\mathcal{S}}; \mathbf{X}_{\mathcal{V}}),
\end{equation}
where $\mathcal{S} \subseteq \mathcal{P}$ denotes the selected subset of path primitives that will be sampled; $\mathcal{V}$ denotes the set of locations of interest. 

Besides the objective of modeling an environmental property, there may also be hypothesis-related objectives. For example, human scientists may hypothesize that the dependence of soil shear strength on moisture can be represented as a piece-wise linear function, as shown in Fig. \ref{fig:hypothesis}, Hypothesis B. After taking a few samples at different locations, we can use GP regression to predict the mean and variance of all locations of interest. Then, we can use the predicted moisture as input together with the hypothesis function to compute a hypothesized soil shear strength~\cite{liu2024understanding}. The difference between the hypothesized value and the GP-predicted value is defined as the discrepancy of that location, i.e.,
\begin{equation*}
    D(v) = |h_j(\mu_m(v))-\mu_s(v)|, ~~\forall v \in \mathcal{V},
\end{equation*}
where $h_j$ denotes the hypothesis function and $\mu_m(v)$  and $\mu_s(v)$ denote the mean value of soil moisture and strength computed using GP regression.

The hypothesis related objective for the set of locations, $\mathcal{V}$, is defined as 
\begin{equation}\label{eq:objective:discrepency}
    g_{j}(\mathcal{S}) = \sum_{s \in \mathcal{S}} \sum_{v \in \mathcal{V}} {D(v)} e^{-\lambda d(s, v)} ~\rm{if} \mathcal{S} \neq \emptyset; ~otherwise ~0,
\end{equation}
where $d(s, v)$ denotes the shortest distance between the primitive $s$ and $v$; $\lambda >0$ is a scaling factor. The intuition of Eq. \eqref{eq:objective:discrepency} is that by maximizing such an objective, we will bias robots towards the nodes with large discrepancy values.

\textbf{{Objective Function}:}
The final objective function is defined as the weighted sum of all objectives defined in Eq. \eqref{eq:objective:mutual_information} and Eq. \eqref{eq:objective:discrepency}, i.e., 
\begin{equation}\label{eq:data_collection_objective}
    f(\mathcal{S}, \bm{\theta}) = \sum_{i} \theta_{i} g_{i}(\mathcal{S}) + \sum_{j} \theta_j g_{j}(\mathcal{S}). 
\end{equation}

\begin{theorem}\label{theorem:submodular_objective_scientific}
    The objective defined in Eq. \eqref{eq:data_collection_objective} is monotone submodular.
\end{theorem}
The proof is given in the Appendix. %\todo{to revise}
%\todo{not very rigorous }

\subsection{Multi-Robot Coverage Control for Event Detection}\label{sec:case_study:coverage}
 This case study on multi-robot multi-objective coverage control is inspired by~\cite{sun2017submodularity, lee2021mobile, zhong2011distributed}.  The goal is to coordinate the motion of robots to detect multiple events as shown in Fig. \ref{fig:illustration_coverage}. Such a problem can be formulated as a maximization problem with a submodular objective, built on the probability of detecting stochastic events.  Such a submodular objective is usually designed offline based on human preferences and prior knowledge of the events. However, human preferences among detecting the various events may change when the human supervisor receives new information online or external instructions (e.g., detect event $i$ and $j$ as soon as possible). In such cases, the human supervisors may give a suggestion indicating which path primitives the robot team should take. Once receiving such suggestions, the robot team will solve the ISM problem to minimally change the objective to adapt to the new suggestion. It should be noted that such suggestions are given only occasionally, not at every planning cycle. When there are no suggestions from humans, the robot team will coordinate by solving the forward optimization problems.

Specifically, there is a team of robots $\mathcal{R}=\{1, \ldots, n_r\}$ in the task environment. Each robot has the sensing capability to detect events, and we aim to coordinate the team to collectively detect independent events $\mathcal{E}=\{1, \ldots, m\}$ in the environment. Each robot has a finite set of available path primitives, each of which spans a planning horizon of length $H$. We consider the action selection problem for the team.  

\textbf{{Environment and Event Models}:}
%When the robotic team makes decisions on what path primitives to take for the next planning epoch to track a group of targets, they need to have a prediction model that can predict the motion of the targets. We assume that in the offline phase, we have collected some motion data of other different groups of targets. We assume that the motion of the group of targets that the robots are going to track shares some similarities compared to that in the data. We assume that we have trained a library of models for representative groups, each of which can make an independent prediction of where the targets will go. We will combine these predictions in the optimization objective.    \todo{should include uncertainty}
The task environment $\Omega \subset \mathbb{R}^2$ is a polygon embedded in 2-dimensional space. 
%There are several non-overlapping polygon obstacles in the environment. The obstacle can block the movement and the sensing of the robots. 
Each event $j \in \mathcal{E}$ is associated with an \textit{event density} function $\phi_j(x): \Omega \to \mathbb{R}_{+}$, which describes the frequency or density of the occurrences of a particular stochastic event $j \in \mathcal{E}$ at the location $x$. Depending on the application, $\phi_j(x)$ can be the frequency that a target shows at $x$, or it can be the probability that some physical quantities, e.g., temperature, humidity, exceed some threshold \cite{li2005distributed}. Generally, $\phi_j(x)$ needs to satisfy two conditions: $\phi_j(x) \geq 0$ and $\int_{\Omega} \phi_j(x) dx < \infty$. In this paper, we are interested in a particular case where $\phi_j$ is a probability distribution over $\Omega$ and $\int_{\Omega} \phi_j(x) dx =1$. When we select path primitives for robots, we need to consider all the events in $\mathcal{E}$. To achieve this, we give each event $j$ an importance factor $\theta_j \in \mathbb{R}_{+}$, which are designed offline, and use a weighted sum approach to balance multiple events. 
We assume that within a planning horizon $H$ the density function $\phi_j(x)$ is time-invariant. %but over a much longer horizon $\phi_j(x)$ may change slowly w.r.t. time. 

\textbf{{Robots and Sensing}:}
The position of robot $i \in \mathcal{R}$ is denoted as $x_i \in \Omega$. 
Each robot carries onboard sensors and has a sensing radius $\delta$. Therefore, the sensing region of robot $i$ is $\Omega_i = \{x \in \Omega \mid \norm{x-x_i} \leq \delta \}$. For a planning horizon $H$, a path primitive $p$, is defined as a sequence of positions that the robot will reach in order, i.e., $p=[x_i(t=k), \ldots, x_i(t=k+H)]$. We assume that a lower-level planner can generate a set of feasible path primitives for robots. 
%If the robot is around the obstacles, part of the sensing region will be occluded. To account for this, we define the \textit{visible set} for the robot $i$. A point $x \in \Omega_i$ is visible from $p_i$ if the line segment defined by $x$ and $p_i$ will not intersect with the obstacles. We denote the visible set of the robot $i$ as $\Omega_i^{V}(p_i) \subseteq \Omega_i$. 
The sensing model of robot $i$ is given as $\text{Pr}(x, x_i)$, the probability that the robot $i$ can detect the event occurrences at $x \in \Omega_i$. If a point is not in the sensing region, the probability is zero. If it is within the sensing region, we assume that  $\text{Pr}(x, x_i)$ is a function of the distance between $x$ and $x_i$, i.e., $\norm{x-x_i}$, and is monotonically decreasing and differentiable. Overall, the sensing model of the robot that we use here is 
\begin{equation}
    \text{Pr}(x, x_i) = 
    \begin{cases}
        \text{exp}(-\xi_i \norm{x-x_i}) ~&\text{if}~ x \in \Omega_i(x_i) \\
        0 ~ &\text{otherwise},
    \end{cases}
\end{equation}
where $\xi_i$ is a decay factor for sensing. 

Combining all the sensing results from all robots, we can compute the joint detection probability that an event at $x \in \Omega$ is detected by the team at time step $t=k$ as, 
\begin{equation}
    \text{Pr}(x, \bm{x}(k)) = 1 - \prod_{i=1}^{n_r} (1-\text{Pr}(x, x_i(k))),
\end{equation}
where $\bm{x}(k)=[x_1(k), \ldots, x_{n_r}(k)]^T$ denotes the position vector of all robots at time $k$.

%At each time step, the robot will receive measurements of targets that show in the field-of-view of sensors. Then the robotic team will fuse the measurements and update the detection and estimation using either Kalman filters or Random Finite Sets (RFS) filters. At the end of the planning epoch, the team will use the most recent estimation results to predict where the targets will go and plan for the next epoch.
%\todo{robot motion primitives and its utilities}

%At the current time $t$, we want to select a trajectory primitive for each robot such that the expected number of detected targets\footnote{Other objective such as mutual information is also applicable here.} in the following epoch is maximized. Each motion prediction model in \todo{XX}, it corresponds to one such submodular objective function. We use the weighted sum of such objectives for the decision-making. 

%the objective can be improved later using results on multiplication of two supermodular objective functions
%https://math.stackexchange.com/questions/1708936/required-conditions-for-product-of-two-submodular-functions-to-be-submodular
%or try to consider the max(f, g)
\textbf{{Objective Functions}:} For time step $k$, we define the event coverage objective of event $j$, as defined in 
\cite{sun2017submodularity}, as 
\begin{equation}\label{eq:single_step_coverage}
    h_j(\bm{x}(k))=\int_{\Omega} \phi_j(x) \text{Pr}(x, \bm{x}(k)) dx.
\end{equation}

 When $\phi_j$ is a probability distribution over $\Omega$, Eq. \eqref{eq:single_step_coverage} yields the probability of detecting the event $j$ if the robot team senses at $\bm{x}(k)$. Therefore, $1-h_j(\bm{x}(k))$ is the probability of not detecting event $j$ at time $k$. We assume the sensing at each time step is independent. The probability of not detecting event $j$ over all $H$ steps is: 
\begin{equation}
    \prod_{t=k}^{t={k+H}} (1-h_j(\bm{x}(t))).
\end{equation}

As a result, the detection probability of an event $j$ over the whole planning horizon is 
$1-\prod_{t=k}^{t={k+H}} (1-h_j(\bm{x}(t)))$.
There are $m$ such events, and each is associated with parameter $\theta_j$ describing its priority.   
Overall, the coverage objective for the team and all events over the whole planning horizon can be expressed as:
\begin{equation}\label{eq:submodular_coverage_objecive}
    f(\mathcal{S}, \bm{\theta}) = \sum_{j=1}^n \theta_j (1-\prod_{t=k}^{t={k+H}} (1-h_j(\bm{x}(t)))),
\end{equation}
where $\mathcal{S}=\{p_1, \ldots, p_{n_r}\}$ is the selected primitive set for the team, robot $i$'s primitive $p_i=[x_i(t=k), \ldots, x_i(t=k+H)]$ consists of a sequence of positions, 
%is the action for the robot $i$. 
$\mathcal{S}$ can be viewed as a $n \times H$ matrix (each element is a 2D tuple) and $\bm{x}(t)$ is a column, and $\theta_j \in \mathbb{R}_{+}$ is the importance factor for detecting event $j$. 

\begin{theorem}\label{theorem:submodular_objective_horizon}
    The objective defined in Eq. \eqref{eq:submodular_coverage_objecive} is monotone submodular.
\end{theorem}
The proof is given in the Appendix. %\todo{to revise}
%\todo{not very rigorous }

\section{Algorithm: Single Human Suggestion}\label{sec:algorithm_single}
%\subsection{Human Suggestion is Complete}
%consider the techniques used in differentiate through black-box solver 
%Baysian optimization

Let us first consider the case where the human suggestion $\hat{\mathcal{S}}$ is an ordered set. Based on Algorithm \ref{algorithm:greedy}, at each step, the element with the largest marginal gain will be selected. If $\hat{\mathcal{S}}$ is an ordered set and is the output of the Algorithm \ref{algorithm:greedy},  for each prefix $\hat{\mathcal{S}}[1:i]$, it should satisfy the following inequality based on line 3 in Algorithm \ref{algorithm:greedy}:
\begin{align}\label{eq:greedy_inequality}
    \begin{aligned}
          f(\hat{\mathcal{S}}[1:i], \hat{\bm{\theta}}) - f(\hat{\mathcal{S}}[1:i-1], \hat{\bm{\theta}})  \geq  \\
          f(\hat{\mathcal{S}}[1:i-1] \cup \{s\}, \hat{\bm{\theta}}) - f(\hat{\mathcal{S}}[1:i-1], \hat{\bm{\theta}}), \\ 
          \forall s \in \{ s \in \mathcal{P} \setminus \hat{\mathcal{S}}[1:i] \mid \{s\} \cup \hat{\mathcal{S}}[1:i-1] \in \mathcal{I} \}.
    \end{aligned}
\end{align}
The left side of the inequality is the marginal gain of adding an element $\hat{\mathcal{S}}[i]$ to the ordered set $\hat{\mathcal{S}}[1:i-1]$. The right side of the inequality is the marginal gain of adding other feasible elements $s$. Intuitively, each  $\hat{\mathcal{S}}[i] \in \hat{\mathcal{S}}$ should be the one with the largest marginal gain in the $i$-th selection step.

For the case that $f$ is a linear function w.r.t. $\hat{\bm{\theta}}$ as shown in Eq. \eqref{eq:linear_submodular}, each inequality \eqref{eq:greedy_inequality} is  a linear inequality:
\begin{equation}\label{eq:linear_inequality}
  \begin{aligned}
        \hat{\bm{\theta}}^{T}g(\hat{\mathcal{S}}[1:i])-\hat{\bm{\theta}}^{T}g(\hat{\mathcal{S}}[1:i-1]) \geq \\
        \hat{\bm{\theta}}^{T}g(\hat{\mathcal{S}}[1:i-1] \cup \{s\}) - \hat{\bm{\theta}}^{T}g(\hat{\mathcal{S}}[1:i-1]).
  \end{aligned}
\end{equation}
There will be $O(\abs{\hat{\mathcal{S}}} \cdot \abs{\mathcal{P}})$ such constraints. Combining all the linear inequalities, the Ordered-set variant of Problem 1 boils down to a convex optimization problem as follows.

\begin{subproblem}[Ordered-ISM (O-ISM)]\label{problem:ordered_ism}
    \begin{align}
    \min_{\hat{\bm{\theta}}}~ &\norm{\hat{\bm{\theta}} - \bm{\theta}} \label{eq:min_obj_fix_order}\\
    \rm{s.t.} \quad & \hat{\bm{\theta}}^T b_j \leq 0,  \forall j  \label{eq:all_linear_inequality}
   % & \hat{\bm{\theta}} \in \bm{\Theta} \label{eq:theta_fix_order},
    \end{align}
where $b_j$ denotes a coefficient vector corresponding to Eq. \eqref{eq:linear_inequality}, and $j$ is the index for all linear inequalities.
\end{subproblem}

It should be noted that we can define an O-ISM problem for any ordered set other than $\hat{\mathcal{S}}$. In the rest of this paper, we will treat O-ISM as a function: for an ordered set $\mathcal{A}$, which is assumed to be returned from the greedy algorithm for some parameter $\hat{\bm{\theta}}$, O-ISM($\mathcal{A}$) denotes the solution to Subproblem 
%problem instance as defined in Problem 
\ref{problem:ordered_ism} for the set $\mathcal{A}$. We will use this in Algorithm \ref{algorithm:BB_algorithm}.

When the human suggestion $\hat{\mathcal{S}}$ is not an ordered set, each possible ordering corresponds to a set of constraints as in Eq. \eqref{eq:all_linear_inequality}. Such a case is more practical in applications: human operators know the solutions based on their expertise and observation, but they do not have the concept of the ordering of a solution set because the ordering of the solution requires reasoning over how the approximate algorithm works, which is not involved in the human decision-making process. All these sets of constraints can be connected using OR logic similar to disjunctive inequalities. We can use the Big-$M$ reformulation technique to formulate the problem as a Mixed Integer Quadratic Programming (MIQP) problem.

\begin{subproblem}[Unordered-ISM (U-ISM)]\label{problem:unordered_ism}
    \begin{align}
    \min_{\hat{\bm{\theta}}}~ &\norm{\hat{\bm{\theta}} - \bm{\theta}} \label{eq:deviation_obj_unordered}\\
    \rm{s.t.} \quad & \hat{\bm{\theta}}^T b_j^k \leq M(1-y_k),  ~\forall j, ~\forall k  \label{eq:sequence_inequality_unordered}\\
    & \sum_k y_k = 1,  \quad y_k \in \{0, 1\}\label{eq:integer_constraints_unordered},
    \end{align}
where $k$ is the index for the possible ordering of $\hat{\mathcal{S}}$, $j$ is the index for all linear inequalities corresponding to an ordering, $y_k$ is a binary variable to indicate which ordering constraint is active, and $M$ is a large enough positive number. 
\end{subproblem}
 In Subproblem \ref{problem:unordered_ism}, we add a binary variable $y_k$ to indicate whether an ordering is active. If $y_k = 1$, then the inequalities for $k$-th ordering is active in Eq. \eqref{eq:sequence_inequality_unordered}, i.e., the RHS of Eq. \eqref{eq:sequence_inequality_unordered} is zero. If $y_k = 0$,  the RHS of Eq. \eqref{eq:sequence_inequality_unordered} is a large number $M$, the inequalities are trivially satisfied. We enforce that there should be only one active ordering in Eq. \eqref{eq:integer_constraints_unordered}.

\begin{figure}
\centerline{\includegraphics[scale=0.6]{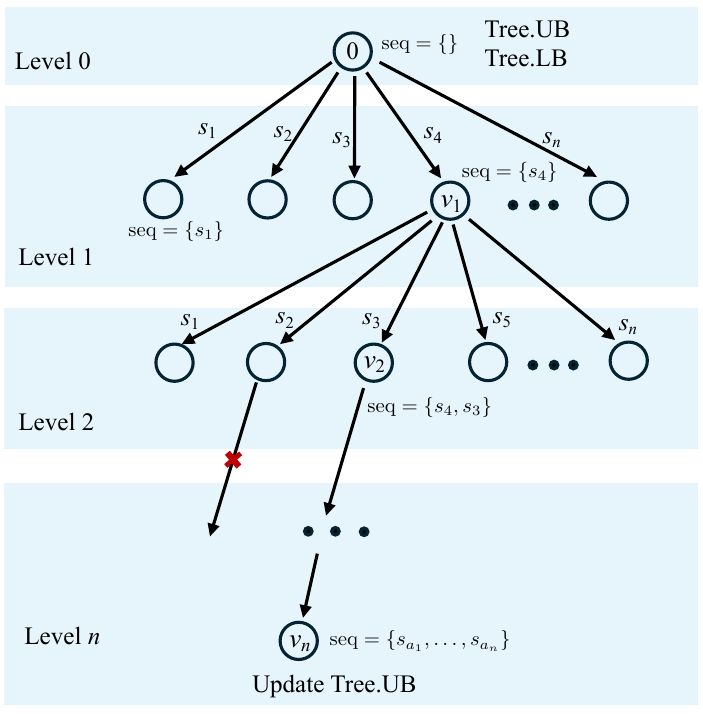}}
\caption{One iteration of the proposed BB-ISM algorithm. Black arrows correspond to node expansion operations, while the red cross corresponds to a pruning operation.}
\label{fig:BB}
\end{figure}

It should be noted that there are exponentially many constraints and integer variables in Eq. \eqref{eq:sequence_inequality_unordered}. The brute-force way to solve the problem is to explicitly list all the constraints and use an existing solver, e.g., Gurobi, to find the solution. However, it is both time-consuming and memory-consuming to explicitly list all the constraints. Besides, the solver would not leverage
%we will lose 
structural information about these binary variables, i.e., that they correspond to the orderings of a set. Therefore, we develop a branch and bound algorithm to solve this type of problem without explicitly listing all the constraints. The main idea is that instead of treating each possible ordering of $\hat{\mathcal{S}}$ as a candidate and searching for the best one, we incrementally add elements to $\hat{\mathcal{S}}$ to form a sequence of relaxed problems. By keeping track of the solutions from these relaxed problems, we can gradually find the upper and lower bounds of the original objective,
%of the original problem, 
and prune the suboptimal solutions. The incremental search is conducted by growing a search tree in depth-first-search fashion. An illustrative example is shown in Fig. \ref{fig:BB}. At the root node, we start with an empty sequence. For the next level (Level 1) of the tree, we can add any element in ${\hat{\mathcal{S}}}$ to the sequence to form a new node. For each such node, we will solve Subproblem O-ISM(\textit{seq}) using the \textit{seq} property of the node, which keeps track of the ordered sequence up to that node starting from the root as illustrated in Fig. \ref{fig:BB}. When the \textit{seq} includes only some of the elements in $\hat{\mathcal{S}}$, the objective value returned by solving O-ISM(\textit{seq}) lower bounds
%can be viewed as the lower bound of 
all cases where the orderings of $\hat{\mathcal{S}}$ start with \textit{seq}, since the ordering with all elements of $\hat{\mathcal{S}}$ has additional constraints. 

\begin{algorithm}[ht]\label{algorithm:BB_algorithm}
% https://tex.stackexchange.com/questions/153646/algorithm2e-disabling-line-numbers-for-specific-lines
    \caption{BB for ISM (BB-ISM)}
    \SetKwInOut{Input}{Input}
    \SetKwInOut{Output}{Output}
    \SetKwProg{Fn}{Function}{:}{}
    \Input{Problem instance with human suggestion $\hat{\mathcal{S}}$
    }
    \Output{$\hat{\bm{\theta}}$}
    %\nonl \# Initialize a search tree \\
    Tree $\gets$ empty tree \quad \# Initialize a search tree \\
    Tree.UB $\gets$ a large number \\
    %Tree.LB $\gets$ $\max_{\hat{\bm{\theta}} \in \bm{\Theta}} \norm{\hat{\bm{\theta}} - \bm{\theta}}$ \\
    Tree.add\_node(node\_ID = 0, sequence = $\{\}$) \\
    \nonl \# Initialize an empty stack for Depth First Search \\
    Stack $\gets$ empty stack \\
    Stack.push(Tree.get\_node(node\_ID = 0)) \\
    \While{Stack is not empty}{
    $u$ $\gets$ Stack.pop() \\
    PQ $\gets$ priority\_queue() \\
    \For{$s \in \hat{\mathcal{S}} \setminus {u\rm{.seq}}$ }
    {
    feasible, obj, $\hat{\bm{\theta}}$ $\gets$ O-ISM($u\rm{.seq} + \{s\}$)\\
        \If{feasible}{
        PQ.insert($s$, $\hat{\bm{\theta}}$, priority\_value=obj)   
        }
    }
    \While{PQ is not empty}{
        $s$, $\bm{\theta}$, obj $\gets$ PQ.pop() \\
        \If{obj $<$ Tree.UB}{
            Tree.update\_UB($u$.seq+$\{s\}$, obj) \\
            %Tree.update\_LB($u$.seq+$\{s\}$, obj) \todo{need to explain why need this}\\
            new\_node $\gets$ Tree.add\_branch($u$, $s$, obj, $\bm{\theta}$) \\
            Stack.push(new\_node)
        } 
    }
    }
    \textbf{return} $tree$ \\
     \SetKwFunction{updateUB}{update\_UB}
    \Fn{\updateUB{$Tree, ~sequence$, obj}}{
    \If{length of $sequence$ = length of $\hat{{\mathcal{S}}}$}
        {
            \If{obj $<$ $Tree$.UB}
            {
                $Tree$.UB $\gets$ obj \\
                $Tree$.UB\_seq $\gets$ $sequence$
            }
        }
    }
    \textbf{end} 
    
\end{algorithm}

\iffalse
\begin{algorithm}[ht]\label{BB_subroutine}
    \caption{B\&B Search Subroutines}
    \SetKwFunction{prune}{Prune\_branch}
    \SetKwProg{Fn}{Function}{:}{}
    \Fn{\prune{Tree, obj}}{
    \eIf{obj $\geq$ Tree.UB}{
        \textbf{return} True
    }
    {    
        \textbf{return} False
    }
    }
    \textbf{end} \\
    
    %\nonl ~ \\
    \SetKwFunction{updateUB}{update\_UB}
    \Fn{\updateUB{$Tree, ~sequence$, obj}}{
    \If{length of $sequence$ = length of $\hat{\bm{\theta}}$}
        {
            \If{obj $<$ $Tree$.UB}
            {
                $Tree$.UB $\gets$ obj \\
                $Tree$.UB\_seq $\gets$ $sequence$
            }
        }
    }
    \textbf{end} \\
 
    %\nonl ~ \\
    \iffalse
    \SetKwFunction{updateLB}{update\_LB}
    \Fn{\updateLB{$Tree, ~sequence$, obj}}{
        {
            \If{obj $<$ $Tree$.LB}
            {
                $Tree$.LB $\gets$ obj \\
                $Tree$.LB\_seq $\gets$ $sequence$
            }
        }
    }
    \textbf{end} 
    \fi
\end{algorithm}
\fi

Then, we will choose the node (node $v_1$ in Fig. \ref{fig:BB}) with the smallest objective value (returned by solving O-ISM(\textit{seq})) for further expansion.
%to further expand on that node. 
Here, we use the objective value of O-ISM(\textit{seq}) as a heuristic to guide search and use a greedy strategy to expand to the next level. The expansion is similar to that from Level 0 to Level 1. There is only one element $s_4$ in the \textit{seq} of the node $v_1$. So we can add any elements in ${\hat{\mathcal{S}}} \setminus \{s_4\}$ to the sequence to form a new child node on Level 2 and solve a corresponding problem O-ISM(\textit{seq}). Such a process will repeat until the \textit{seq} of the node includes all elements in ${\hat{\mathcal{S}}}$, which means that an ordering of ${\hat{\mathcal{S}}}$ is found. We will use the objective value, returned by solving O-ISM(\textit{seq}), of this node to update the upper bound of the problem (Tree.UB in the root node), since this solution is for a particular ordering of ${\hat{\mathcal{S}}}$, while the optimal solution corresponding to an optimal ordering minimizes the objective under the constraints.
%should achieve a lower objective value. 

After growing the tree to a full candidate ordering of $|\hat{\mathcal{S}}|$, we 
return to the previous node level and expand,
%will continue to grow the tree by going back to the previous level of the node and expanding, 
similar to Depth First Search (DFS). As we grow the search tree, if a particular O-ISM(\textit{seq}) problem is infeasible, which implies all orderings with prefix \textit{seq} are infeasible, or if the objective value returned from O-ISM(\textit{seq}) exceeds the tree's upper bound identified so far, which implies that no ordering with prefix \textit{seq} can outperform the best solution identified so far, we will prune that branch as shown in Fig. \ref{fig:BB} (the red cross prunes the branch). Such pruning operations accelerate the search.

The details are shown in the Algorithm \ref{algorithm:BB_algorithm}. In lines 1-3, we initialize an empty tree with a root node whose \textit{seq} property is an empty ordered set. The upper bound property is initialized to a large number, and will be updated as the search tree grows. In lines 4-5, we initialize a stack for a DFS-style search and push the root node to the stack. In the while loop, we first pop the top element from the stack (line 7) and initialize a priority queue (line 8). Then each element $s \in \hat{\mathcal{S}}$ that is outside $u$.seq (line 9) is checked for feasibility to solve a relaxed O-ISM problem if appended to the existing sequence (line 10). If infeasible, then no sequence with such a prefix is possible, so we prune this branch by ignoring it. If feasible, we insert element $s$ and the corresponding $\hat{\bm{\theta}}$ into the priority queue, using the objective value given by O-ISM as the priority value. After this, we update the search tree (lines 15-23), first expanding branches in increasing order of the objective value (line 16). For each candidate, we first check whether to prune it (line 17) by comparing the objective value with the tree's upper bound. If the objective value exceeds the upper bound, we prune branches with this sequence as prefix (i.e., they are ignored in line 17). This is because the relaxed O-ISM problem returns a lower bound on the objective value over all orderings of $\hat{\mathcal{S}}$ with prefix $u$.seq: the complete ordering adds more constraints, increasing the objective. % will be greater than this objective value. T
Importantly, the tree's upper bound represents the best solution found so far, while any other orderings generating solutions worse than the identified ones should be pruned. By contrast, if the objective does not exceed the tree's upper bound
%if we should not prune this branch
(line 17), we update the search tree's upper bound and add the new branch to the tree (line 19). The DFS-style search pushes the new nodes to the stack (line 20). After considering all elements in the priority queue, we continue the while loop, popping the element at the top of the stack (line 7) and repeating the process.      

\begin{theorem}
    Given a feasible problem instance as described in Subproblem \eqref{problem:unordered_ism}, Algorithm \ref{algorithm:BB_algorithm} returns the optimal solution in finite iterations of the outer while loop (lines 6-24).
 \end{theorem}
Like all algorithms developed within the BB paradigm, the BB-ISM algorithm, by nature, enumerates all possible combinations by incrementally adding elements. The efficiency relies on the pruning steps to remove unnecessary expansion of the search tree.

\begin{remark}
    It should be noted that Subproblem \ref{problem:unordered_ism} is an NP-hard problem in general (mixed integer quadratic programming). In the worst case, the proposed algorithm will have to enumerate all possible orderings of the human suggestion to find the optimal solution or find that the problem is infeasible. We will experimentally compare BB-ISM with the brute-force approach using Gurobi in Section \ref{sec:experiment}. Moreover, the approach developed in this section cannot be directly transferred to solve Problem \ref{prob:multiple_suggestion} (multiple human suggestions) for the following reasons. First, we cannot define a subproblem like O-ISM(\textit{seq}) in Problem \ref{prob:multiple_suggestion} for pruning. Eq. \eqref{eq:multiple_suggestion_objective_value} and Eq. \eqref{eq:multiple_suggestion_set_distance} require a complete ordered set for evaluation. Moreover, because there are multiple objectives in Problem \ref{prob:multiple_suggestion}, the pruning criterion in Algo. \ref{algorithm:BB_algorithm} cannot be used. 
\end{remark}

\section{Algorithm: Multiple Human Suggestions}\label{sec:algorithm_multiple}
This section presents a Pareto MCTS-based solution to solve Problem \ref{prob:multiple_suggestion}. The insight is that Problem \ref{prob:multiple_suggestion} can be viewed as a sequential decision-making problem to find an ordered set ${\mathcal{S}}(\hat{\bm{\theta}})$ (defined in Eq. \eqref{eq:multiple:greedy_selection}). Given such an ordered set, the second and third objectives in Problem 2 together with Eq. \eqref{eq:multiple_suggestion_objective_value} and \eqref{eq:multiple_suggestion_set_distance} are determined, since these no longer depend on $\hat{\bm{\theta}}$ given an order of selected elements. We only need to focus on the first objective and Eq. \eqref{eq:multiple:greedy_selection}, which turns out to be an instance of Subproblem \ref{problem:ordered_ism}. Therefore, as long as we can find the Pareto optimal ordered set ${\mathcal{S}}(\hat{\bm{\theta}})$, the problem is solved. However, the brute-force integer programming based approach to finding such an ordered set involves enumerating all possible combinations of elements and their orderings, which is computationally inefficient. In the following, we first explain concepts related to Pareto MCTS~\cite{chen2019pareto, wang2012multi} and then show how to leverage Pareto MCTS to approximately find such an ordered set ${\mathcal{S}}(\hat{\bm{\theta}})$. 

\subsection{Pareto Optimality}
Here we use the same definitions as those defined in \cite{chen2019pareto}. 
\begin{definition}[Pareto Dominance]
Let $\mathbf{Z}_k \in \mathbb{R}^D$ be the $D$-dimensional reward vector associated with the choice $k$, and let $Z_{k,d}$ denote its $d$-th component.  
We say that choice $i$ \emph{dominates} choice $j$, written $i \succ j$ or $j \prec i$, if and only if the following conditions hold:
\begin{enumerate}
    \item $Z_{i,d} \ge Z_{j,d}$ for all $d \in \{1,\ldots,D\}$; and
    \item There exists at least one dimension $d \in \{1,\ldots,D\}$ such that $Z_{i,d} > Z_{j,d}$.
\end{enumerate}

If only condition (1) is satisfied, we say that choice $i$ \emph{weakly dominates} choice $j$, denoted $i \succeq j$ or $j \preceq i$.

In some cases, neither $i \succeq j$ nor $j \succeq i$ holds.  
Then we say that choice $i$ is \emph{incomparable} with choice $j$, written $i \parallel j$, if and only if
there exist two dimensions $d_1, d_2 \in \{1,\ldots,D\}$ such that
$Z_{i,d_1} > Z_{j,d_1}$ and $Z_{i,d_2} < Z_{j,d_2}$.

Finally, choice $i$ is said to be \emph{non-dominated} with respect to choice $j$ (denoted $i \neq j$ or $j \not\succ i$) if and only if there exists at least one dimension $d$ such that $Z_{i,d} > Z_{j,d}$.
\end{definition}

\begin{definition}[Pareto-Optimal Node Set]
Let $\mathcal{V}$ be a set of nodes. A subset $\mathcal{P}^\star \subset \mathcal{V}$ is the
\emph{Pareto-optimal node set} (with respect to expected reward) if and only if
\[
\left\{
\begin{aligned}
&\forall\, v_i^\star \in \mathcal{P}^\star \ \text{and}\ \forall\, v_j \in \mathcal{V},\quad
v_i^\star \not\prec v_j,\\[2pt]
&\forall\, v_i^\star, v_j^\star \in \mathcal{P}^\star,\quad
v_i^\star \parallel v_j^\star .
\end{aligned}
\right.
\]
Here, $a \prec b$ denotes that $b$ (strictly) dominates $a$, and $a \parallel b$ denotes that $a$ and $b$ are incomparable.
\end{definition}

\begin{algorithm}[ht]\label{MCTS_algorithm}
% https://tex.stackexchange.com/questions/153646/algorithm2e-disabling-line-numbers-for-specific-lines
    \caption{Pareto Monte Carlo Tree Search}
    \SetKwInOut{Input}{Input}
    \SetKwInOut{Output}{Output}
    \SetKwProg{Fn}{Function}{:}{}
    \SetKwFunction{FMain}{Pareto-MCTS}
    \Input{Instance of Problem \ref{prob:multiple_suggestion}}
    \Output{An ordered set}
    $\mathcal{S} \gets \emptyset$, $tree ~\gets empty$ \\
    \While{$\abs{\mathcal{S}} \neq n_r$}{
        $tree~ \gets$ \texttt{Pareto-MCTS}($tree$)  \\
        $s ~\gets $ SelectChild($tree$) \\
        $\mathcal{S} \gets \mathcal{S} \cup \{s\}$ \\
        $tree~ \gets$ PruneTree($tree$)
    }
    \textbf{return} $\mathcal{S}$ \\
    \Fn{ \FMain{tree} }{
    \While{computational budget not used up}{
    $v_{sel} \gets ~\text{Selection}(tree, ~v_0)$ \\
    $v_{exp} \gets ~\text{Expansion}(tree, ~v_{sel})$ \\
    $Reward \gets ~\text{Simulation}(tree, ~v_{exp})$ \\
    $\text{Backpropagation}(tree, ~Reward, ~v_{exp})$
    }
    \textbf{return} $tree$
    }
    \textbf{end} \\
\end{algorithm}

\begin{figure}
\centerline{\includegraphics[scale=0.36]{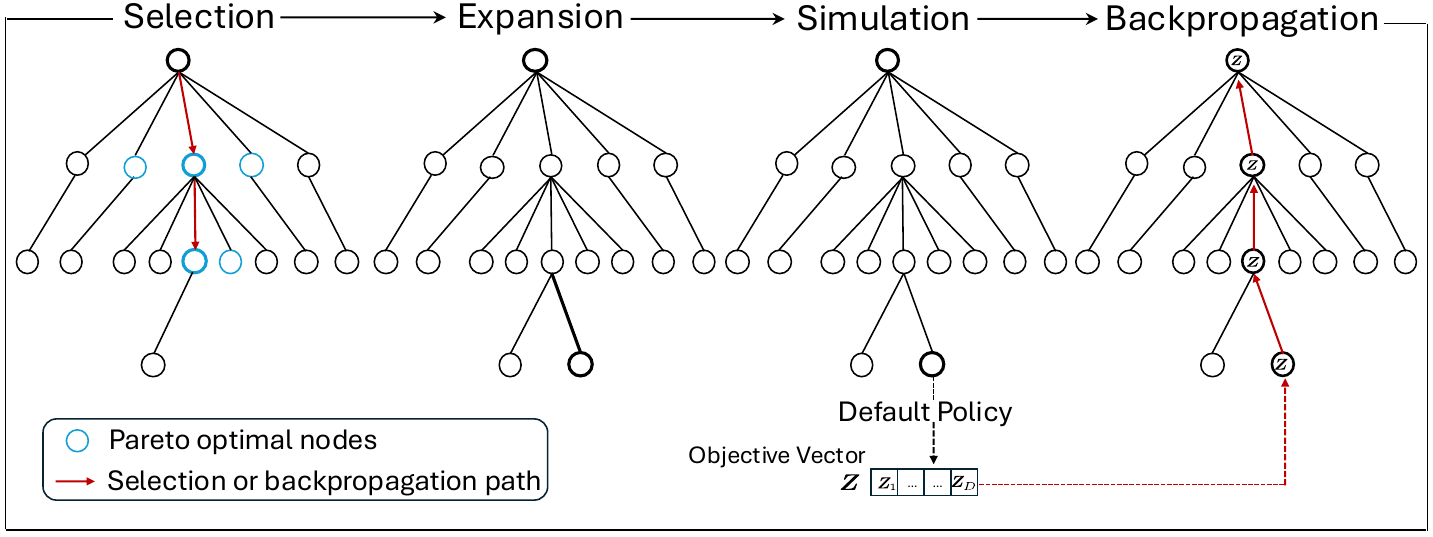}}
\caption{One iteration of the Pareto MCTS.
}
\label{fig:MCTS}
\end{figure}
\subsection{Pareto Monte Carlo Tree Search}
Monte Carlo Tree Search is a heuristic decision-making algorithm widely used in game-playing AI, planning, and optimization problems~\cite{browne2012survey}. It is also widely used in robotics applications, including routing~\cite{zhang2021game, shi2023robust}, active parameter estimation~\cite{slade2017simultaneous}, environment monitoring~\cite{marchant2014sequential}, and multi-robot active perception~\cite{best2019dec}.
It builds a search tree incrementally by simulating many random rollouts from possible action choices and using the outcomes to estimate the value of each choice. Through repeated exploration, MCTS balances trying promising actions (exploitation) and exploring less-visited options (exploration), typically using the Upper Confidence Bound (UCB) formula~\cite{browne2012survey}. This balance allows the algorithm to focus computational effort on the most relevant parts of the search space. As shown in Fig. \ref{fig:MCTS}, there are four basic steps in each iteration of an MCTS process: selection, expansion, simulation, and backpropagation.  

Whereas the classic MCTS algorithm is designed to maximize one reward objective, Problem \ref{prob:multiple_suggestion} has multiple objectives. We therefore adapt its Pareto variant to find the Pareto optimal decisions~\cite{chen2019pareto, wang2012multi}. The main idea is to incrementally construct an ordered set ${\mathcal{S}}(\hat{\bm{\theta}})$ as defined in Eq. \eqref{eq:multiple:greedy_selection}. Starting from an empty set, i.e., ${\mathcal{S}} (\hat{\bm{\theta}})=\emptyset$ (line 1, Algo. \ref{MCTS_algorithm}), we add one element to it during each step. To do this, the algorithm will incrementally grow the search tree with some computational budget (line 3, Algo. \ref{MCTS_algorithm}) and then select
one candidate element from the Pareto optimal child set of the root node using
the average reward of each child (lines 4-5, Algo. \ref{MCTS_algorithm}). Then, we prune other branches of the search tree and only keep the selected branch (line 6, Algo. \ref{MCTS_algorithm}). If the size of ${\mathcal{S}}(\hat{\bm{\theta}})$ is still less than $n_r$ (line 2, Algo. \ref{MCTS_algorithm}), we continue to grow the pruned search tree in the next round to select the next element to add. Such a process continues until the size of  ${\mathcal{S}}(\hat{\bm{\theta}})$ is $n_r$. Details of the subroutines (Algo. \ref{MCTS_subroutine}) used in the search process are explained below.

\begin{algorithm}[ht]\label{MCTS_subroutine}
    \caption{Monte Carlo Tree Search Subroutines}
    \SetKwFunction{Selection}{Selection}
    \SetKwProg{Fn}{Function}{:}{}
    \Fn{\Selection{$tree, ~v$}}{
    \While{$v$ is fully expanded}{
       compute  Pareto UCB for each child $k$:
       \begin{equation}
           \bm{U}(k) = \frac{v_k.\bm{Z}}{v_k.n}+c\sqrt{\frac{4\ln n + \ln D}{2v_k.n}}.
       \end{equation} \\
       construct local approximate Pareto optimal node set $v.\mathcal{P}$ using $\bm{U}(k)$. \\
       sample a node $v_{sel}$ from $v.\mathcal{P}$ uniformly 
    }
    \textbf{return} $v$
    }
    \textbf{end} \\
    \nonl ~ \\
    \SetKwFunction{Expansion}{Expansion}
    \Fn{\Expansion{$tree, ~v$}}{
    \If{level($v$) != TERMINAL}{
    Add an unexpanded child node $v^{\prime}$ of $v$ to $tree$ \\
    $v \gets v^{\prime}$
    }
     \textbf{return} $v$
    }
    \textbf{end} \\
    \nonl ~ \\
    \SetKwFunction{Simulation}{Simulation}
    \Fn{\Simulation{$tree, ~v$}}{
    \While{level($v$) $\neq$ TERMINAL}{
      $v \gets$ DefaultPolicy($v$)
    }
     compute objective vector $\bm{Z}$ \\
    \textbf{return} -$\bm{Z}$
    }
    \textbf{end} \\
    \nonl ~ \\
    \SetKwFunction{Backpropagation}{Backpropagation}
    \Fn{\Backpropagation{$tree, ~Reward,~v$}}{
    \While{$v \neq $ NULL}{
    \nonl // update total reward value
    $tree.v.\bm{U} \gets tree.v.\bm{U} + Reward$ \\ 
    $tree.v.n \gets tree.v.n + 1$ \\
    
    }
    }
    \textbf{end}
\end{algorithm}

\noindent \textbf{\textit{Selection}} (line 11 in Algo. \ref{MCTS_algorithm}; lines 1-8 in Algo. \ref{MCTS_subroutine}): Starting from the root node, a selection procedure is recursively applied until some leaf node is reached. In each recursion, a child node is selected based on Pareto UCB, which extends the scalar case proposed by Kocsis and Szepesvári \cite{kocsis2006bandit} to the vector case. The UCB value is a summation of two terms: exploitation and exploration. The exploitation term corresponds to the average rollout reward obtained, and the exploration term is decided by the number of times that the node has been visited ($n(v^{\prime})$) and the number of times that the current node's parent has been visited (i.e., $n=\sum_k v_k.n$ where $v_k$ denotes the node corresponding to child $k$). If a node is less visited, the exploration value will increase, encouraging the selection of that node. 
%It should be noted that since we are trying to minimize multiple objectives in Problem \ref{prob:multiple_suggestion}, the exploration term will be designed to have negative values.  

\noindent \textbf{\textit{Expansion}} (line 12 in Algo. \ref{MCTS_algorithm}; lines 9-15 in Algo. \ref{MCTS_subroutine}): One (or more) child nodes are added to the tree based on the available actions. If the node is at the terminal level, e.g., no more action budget, the current node is returned (lines 10-13). Otherwise, the expansion step adds one of the current node's children to the tree and returns that child node (line 14).

\noindent \textbf{\textit{Simulation}} (line 13 in Algo. \ref{MCTS_algorithm}; lines 16-22 in Algo. \ref{MCTS_subroutine}): A forward rollout is conducted from the chosen node using the default policy until it reaches a terminal node. The obtained reward can be computed as follows. Each terminal node corresponds to a particular ordered set ${\mathcal{S}}(\hat{\bm{\theta}})$. We can use this ordered set to compute the equality constraints Eq. \eqref{eq:multiple_suggestion_objective_value} and \eqref{eq:multiple_suggestion_set_distance} and therefore the last two objectives. After this, we need to consider the first objective and Eq. \eqref{eq:multiple:greedy_selection}, which is an instance of Subproblem \ref{problem:ordered_ism} and can be solved easily by convex optimization solvers. Then, we will get a vector for each objective value, and we will return the negated value of this vector. 

\noindent \textbf{\textit{Backpropagation}} (line 7 in Algo. \ref{MCTS_algorithm}; lines 23-27 in Algo. \ref{MCTS_subroutine}): The simulation result is propagated back to the root to update node information along the propagation path.

\section{Experiments}\label{sec:experiment}
%\todo{robot action generation; collision; fixed size of action set}
We validate the proposed ISM formulation and the BB-ISM algorithm in two case studies on multi-robot scientific data collection and multi-objective coverage control for event detection as described in Sec. \ref{sec:case_studies}. We will first present illustrative examples to show how the ISM and human suggestions affect the behaviors of robots. Then, we will quantitatively analyze the ISM framework's adaptation to human preference by comparing with baseline approaches. Moreover, we evaluate the proposed algorithm (BB-ISM) in terms of its optimality, running time, and peak memory usage as the problem size increases.

\subsection{Multi-Robot Scientific Data Collection}
\begin{figure*}[ht]
    \centering
    \subfloat[]{
    \includegraphics[width=0.295 \textwidth]{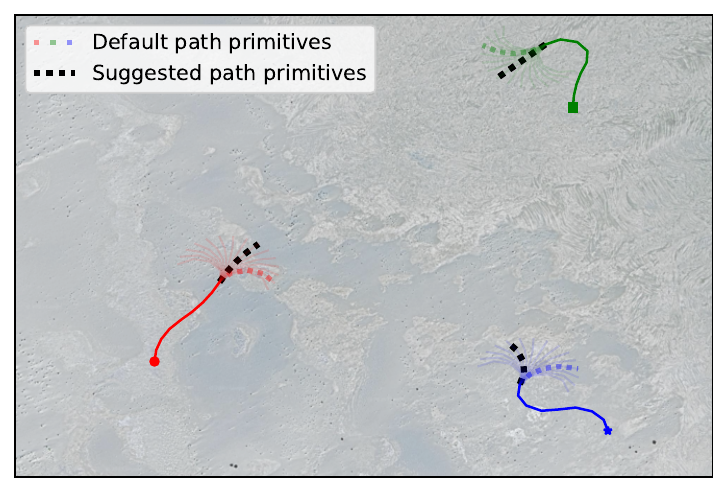}
    \label{fig:first_suggestion}
    } 
    \subfloat[]{
    \includegraphics[width=0.295\textwidth]{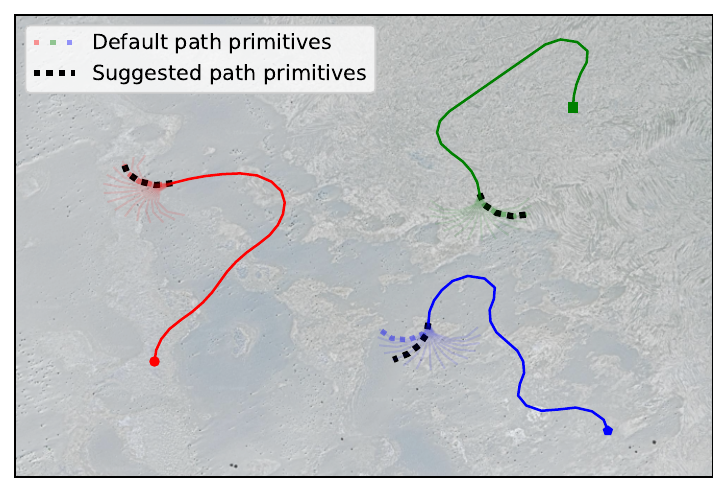}
    \label{fig:second_suggestion}
    }
    \subfloat[]{
    \includegraphics[width=0.295\textwidth]{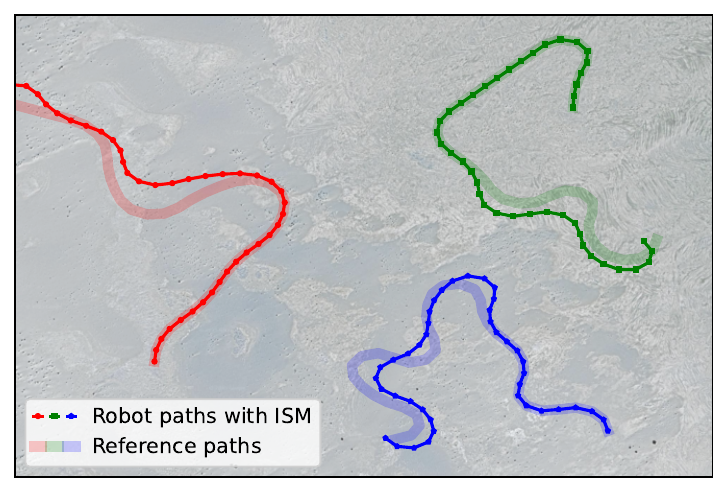}
    \label{fig:human_ISM}
    } \\
    \subfloat[]{
    \includegraphics[width=0.295 \textwidth]{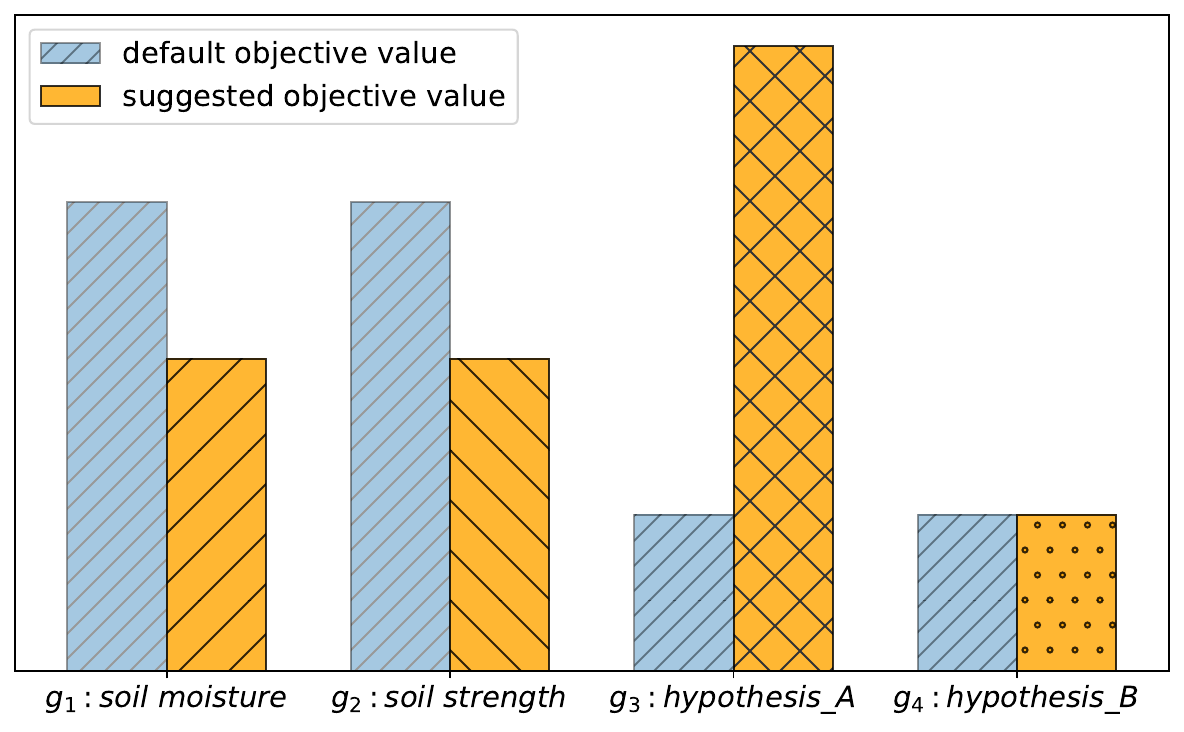}
    \label{fig:first_obj_change}
    } 
    \subfloat[]{
    \includegraphics[width=0.295\textwidth]{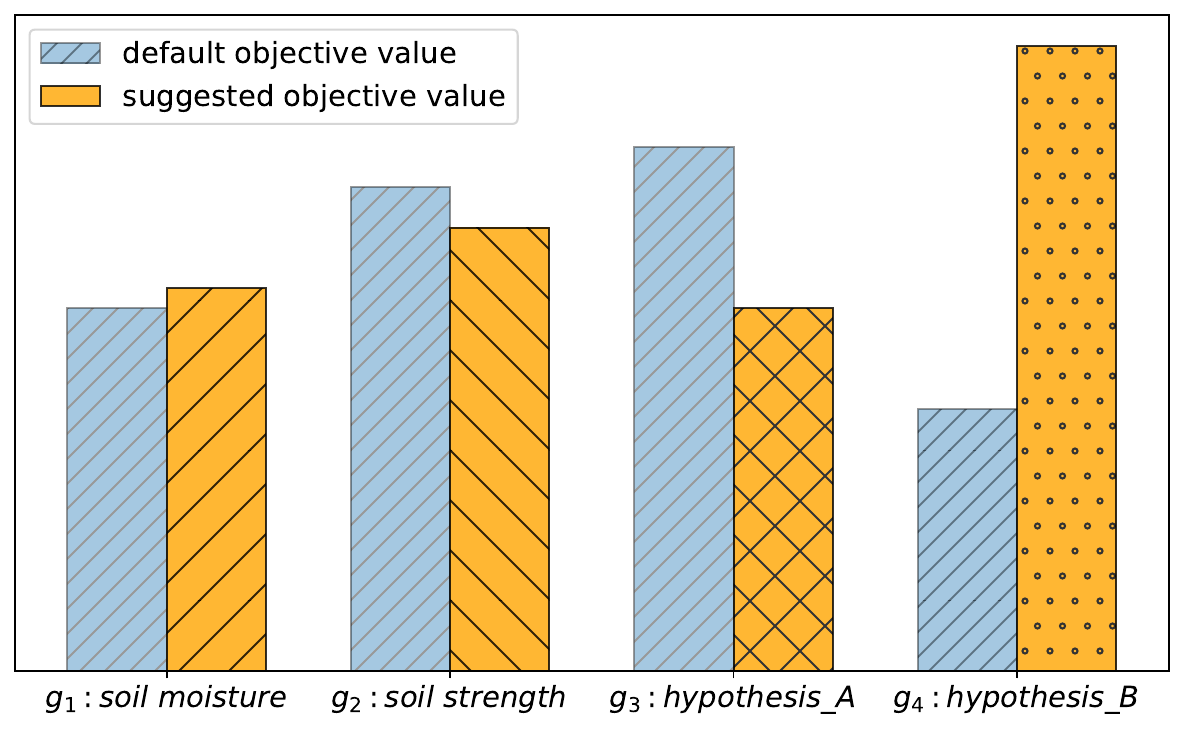}
    \label{fig:second_obj_change}
    }
    \subfloat[]{
    \includegraphics[width=0.295\textwidth]{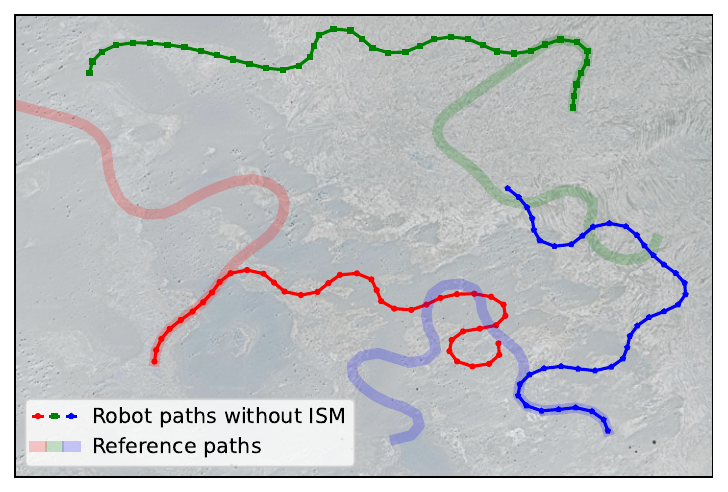}
    \label{fig:human_naive}
    }
    \caption{
     A qualitative example to illustrate how ISM can be used in human-in-the-loop multi-robot scientific data collection. (a) Human makes the first suggestion. Faded lines denote candidate path primitives for robots. Dotted lines in red/green/blue denote the selected primitives using default parameters. Black dotted lines denote the human suggestion. (b) Human makes the second suggestion. (c) The resulting robots' paths with ISM after two suggestions. (d) and (e) Comparisons of four objectives between the default selection and the first and second human suggestions, respectively. (e) The resulting robots' paths without ISM after two suggestions. }
    \label{fig:qualitative_example_information_gathering}
\end{figure*}
\begin{figure}[ht]
    \centering
    \subfloat[]{
    \includegraphics[width=0.295 \textwidth]{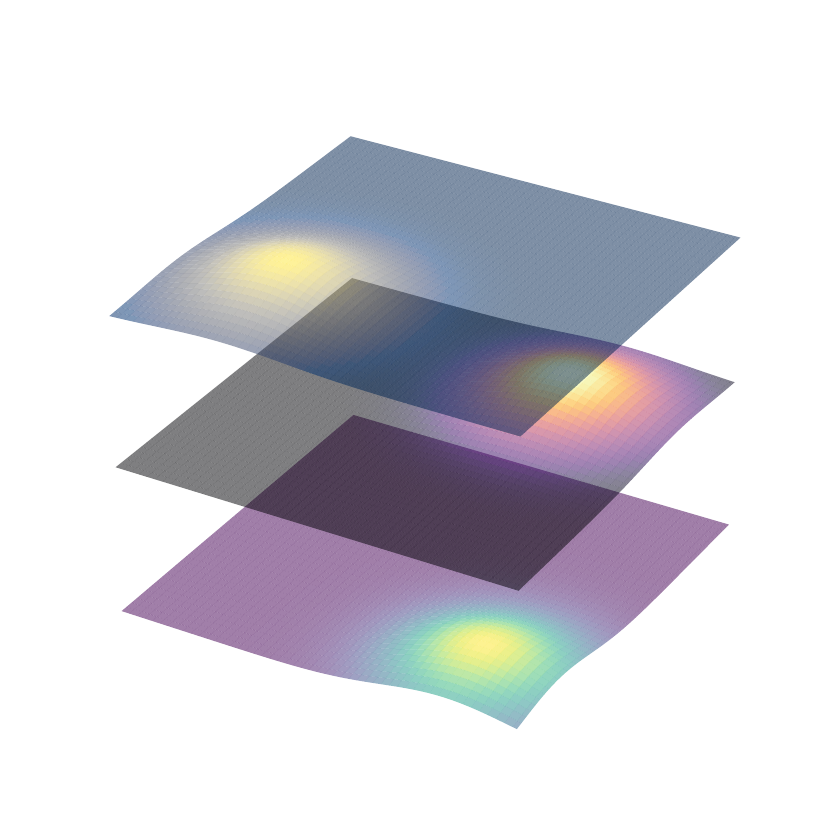}
    \label{fig:events_distribution}
    } 
    \subfloat[]{
    \includegraphics[width=0.175\textwidth]{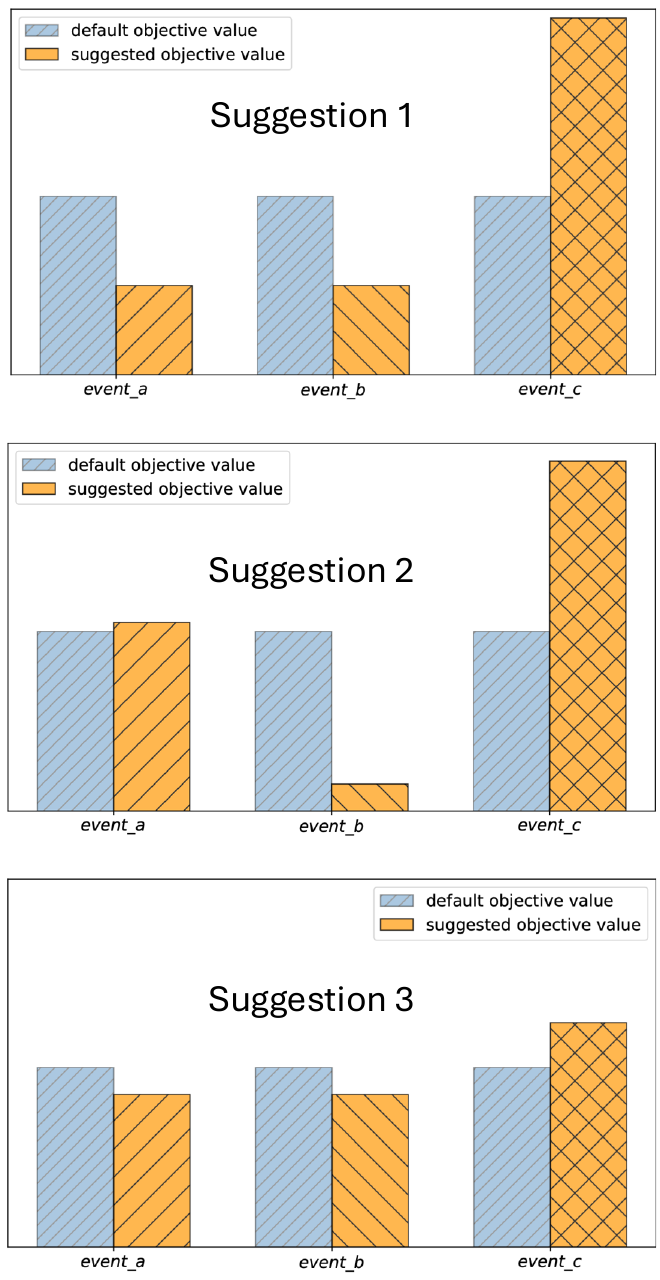}
    \label{fig:suggestion1} 
    } \hfill\\
    \subfloat[]{
    \includegraphics[width=0.305\textwidth]{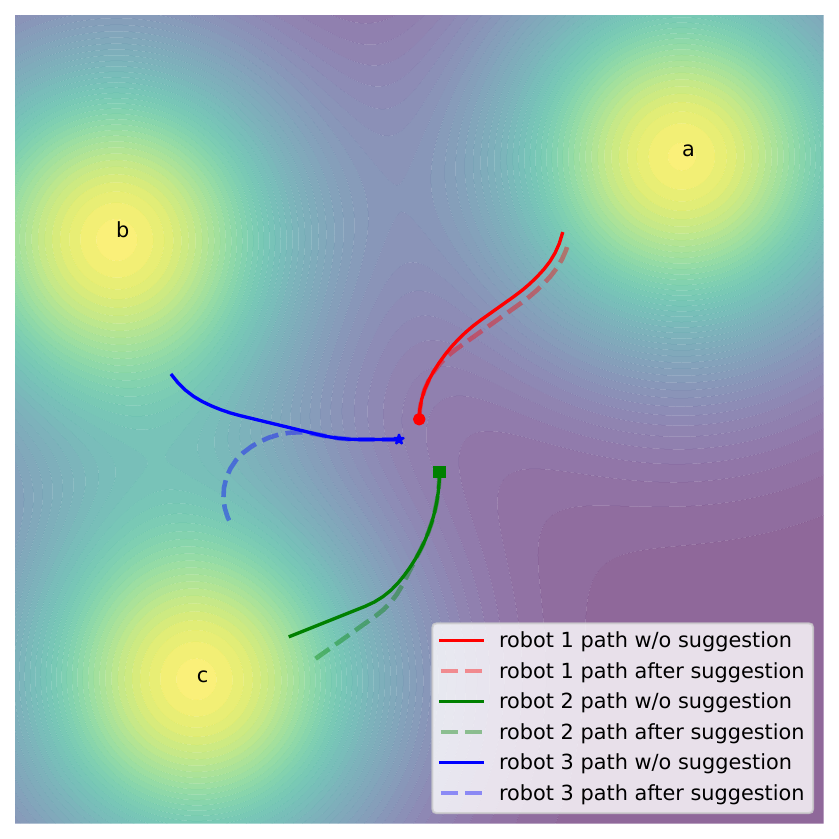}
    \label{fig:traj_with}
    }
    \caption{
     A qualitative example to illustrate how ISM can be used in the human-in-the-loop multi-robot coverage control. (a) Three event density functions. (b) Comparison of the three objective values between the default selection and the three human suggestions, respectively.  The first and third suggestions reflect that event c is more important than a and b; the second suggestion reflects that event c and a are more important than b. (c) Robots' paths w./ w.o. human suggestions. Dotted lines denote paths after suggestions.}
    \label{fig:ISM_qualitative_example_coverage}
\end{figure}
According to studies in multi-robot, multi-objective data collection, human supervisors may change their preferences over objectives during robot data collection~\cite{liu2024modelling, liu2024understanding}. In the following experiments, we assume that we know the ground truth preference of humans when they give suggestions to the robotic team. We will compare how the team performs after solving ISM compared to such ground truth behavior. A qualitative example is shown in Fig. \ref{fig:qualitative_example_information_gathering}. The problem setup is the same as that described in Sec. \ref{sec:case_studies:scientific}. There are 40 points of interest in the environment. The goal is to coordinate robots to take measurements to reduce the uncertainty estimation of properties measured at those points and test some hypotheses. Specifically, each robot can measure two physical properties of the soil: moisture and soil strength. Each property over the task area is assumed to be approximately represented as a GP with known parameters (radial basis function as kernel). There are two hypotheses to be tested: a) the relation between soil moisture and strength can be described by a piecewise linear monotone increasing function, as shown in Fig. \ref{fig:hypothesis}, and b) the moisture and strength can be described by a first-increase-then-decrease function. There are four objectives: two uncertainty metrics for GP of soil moisture and strength as defined in Eq. \eqref{eq:objective:mutual_information} (Property Objective A and B); two hypothesis-related objectives as defined in Eq. \eqref{eq:objective:discrepency} (Hypothesis Objective A and B). We follow the conclusion from  \cite{liu2024modelling, liu2024understanding} and assume that the human will first equally prioritize the two soil property uncertainty metrics, while after robots collect some initial data, the human will then start to focus more on the Hypothesis Objective A, and then on Hypothesis Objective B. That is, during the sampling process, humans will make suggestions twice: the first time, they will show a preference shift towards testing the first hypothesis, and the second suggestion reflects their concern to validate the second hypothesis.

Fig. \ref{fig:qualitative_example_information_gathering} shows a complete example of such a process. In Fig. \ref{fig:first_suggestion}, humans made their first suggestion to guide robots towards hypothesis testing: the black dotted lines denote the human suggestion; the dotted red/green/blue lines denote the default selected primitives. The corresponding change in the four objectives is given in Fig. \ref{fig:first_obj_change}: the suggested path primitives result in a high value in the objective related to hypothesis A. 
 In Fig. \ref{fig:second_suggestion}, the human makes another suggestion to shift their attention to the second hypothesis. The corresponding change in objective values is given in Fig. \ref{fig:second_obj_change}. The resulting paths after two suggestions and the reference paths computed using human ground truth is shown in Fig. \ref{fig:human_ISM}: by solving ISM, the robots' sampling paths are well-aligned with what humans are thinking, showing the effectiveness of the ISM to adapt to human preferences. By contrast, without ISM as shown in Fig. \ref{fig:human_naive}, the resulting sampling path will greatly deviate from the ground truth sampling paths.

\subsection{Multi-Robot Coverage Control for Event Detection}
In the multi-robot coverage control scenario, three stochastic events may happen in the environment $\Omega$. Each event $j$ is associated with a Gaussian event density $\phi_j$ as shown in Fig. \ref{fig:events_distribution}. Each robot has 15 path primitives. 
In the offline design phase, the three events are considered to be equally important, i.e., they have the same importance factor in Eq. \eqref{eq:submodular_coverage_objecive}. Correspondingly, the three objectives shown as blue bars in Fig. \ref{fig:suggestion1} have the same height. As a result, when we deploy robots for coverage control, they will have paths plotted as solid lines in Fig. \ref{fig:traj_with}, i.e., they will be guided toward the regions with high event density, and each will move toward one specific region with high event density. 
However, if the human thinks some events are more important and suggests the robot choose different path primitives as shown in Fig. \ref{fig:suggestion1} (three suggestions from humans: the first and third suggestions indicate that event c is more important than a and b; the second suggestion indicates that event c and a are more important than b.), 
the robots will solve ISM as defined in Problem \ref{prob:multiple_suggestion} to update the importance factors based on such a suggestion and change their behavior correspondingly. The resulting paths are shown as dotted lines in Fig. \ref{fig:traj_with}.

%When the robots move based on the offline designed objective, the two human supervisors at some time give two sets of suggestions denoted as black arrows as shown in Fig. \ref{fig:traj_with} and Fig. \ref{fig:traj_with}. If robots only consider the suggestion from the first human, the robots will increase the importance factor of the event, which has a higher density in the suggested direction. Subsequently, the robots will make decisions using updated importance factors, and the resulting trajectories will change as shown in Fig. XX. Likewise, if robots only consider the suggestion from the second human, the resulting paths are shown as XX in Fig. XX. When robots have to consider both sets of suggestions and solve the ISM problem as defined in Eq. XX, the results paths are shown in Fig. XX. In this example, the second set of suggestions has a higher score, and the resulting paths are closer to the case with only the second set of suggestions, but with deviations.

\subsection{Quantitative Analysis} 
To validate the effectiveness of the proposed ISM framework in adapting to changes in human preference, we test our framework in a synthetic dataset generated using the multi-objective coverage setup and compare it with several baselines. Specifically, we randomly generate a set of ground truth human preference parameters, which are used to generate suggestions. For the single suggestion case, we compute the sampling paths using the ground truth parameters as reference paths. %With such reference paths, we will compare 
The paths generated by all approaches, including ISM, are compared with these reference paths. Specifically, we use the normalized deviation between paths as the metric to compute proximity between paths: %for each robot to compare different approaches:
\begin{equation*}
    dist({p}_{ref}, {p}) = \sum_{i=1}^{\abs{{p}_{ref}}} \norm{{p}_{ref}[i]-{p}[i]} / \norm{{p}_{ref}},
\end{equation*}
where ${p}_{ref}$ and ${p}$ denote the human sampling path and algorithm-generated path, respectively.  Each path consists of a sequence of sampling locations with a fixed number of points, where  $\abs{{p}_{ref}}$ and $\norm{{p}_{ref}}$ denote the number of sampling points and the length of ${p}_{ref}$, respectively.  This metric can quantify the deviataion of a path ${p}$ from the reference path ${p}_{ref}$. 

We consider three baseline approaches. The first one is a naive baseline (Naive), in which human suggestions are ignored. The second one is a greedy baseline (GreedyNeighbor), in which we incrementally increase/decrease each parameter in ${\bm{\theta}}$ by a certain amount (0.02 for each step and 0.2 for the total) and pick the best combination of all the parameters that minimizes the overall deviation from human suggestions. 
The third baseline is a stochastic baseline (Stochastic), in which we discretize the parameters in the parameter space ($(0, 10]$ for each parameter in $\hat{\bm{\theta}}$) with a step size of 0.1 and randomly sample a set of combinations, and pick the best one w.r.t. human suggestions. The results are shown in Fig. \ref{fig:deviation_comparison}, with our approach (BB-ISM) plotted in blue. Overall, BB-ISM induces much smaller deviations compared to other three baselines. Furthermore, the deviation increases only mildly as the number of robots increases. 

For the case with multiple suggestions, we randomly generate 200 instances and compare the solution quality w.r.t. the three objectives of Problem \ref{prob:multiple_suggestion}. The results are shown in Fig. \ref{fig:multi_objective_comparison}. For each instance, we use the results (three objective values) obtained by the GreedyNeighbor baseline as an anchor point, used to normalize the results obtained by other approaches. As a result, we set this anchor point at location (1, 1, 1) in the plot. If a certain approach results in a solution that is better than the GreedyNeighbor w.r.t. all three objectives, the solution will be a point within the box $0< x, y, z \leq1$. As we can see from Fig. \ref{fig:multi_objective_comparison},  for most test instances, solutions obtained using our approach are in the desired box region, which implies smaller values in all three objectives. By contrast, solutions by the Stochastic approach scatter far from the box region, while Naive generates solutions spreading even further.

\begin{figure}
\centerline{\includegraphics[scale=0.45]{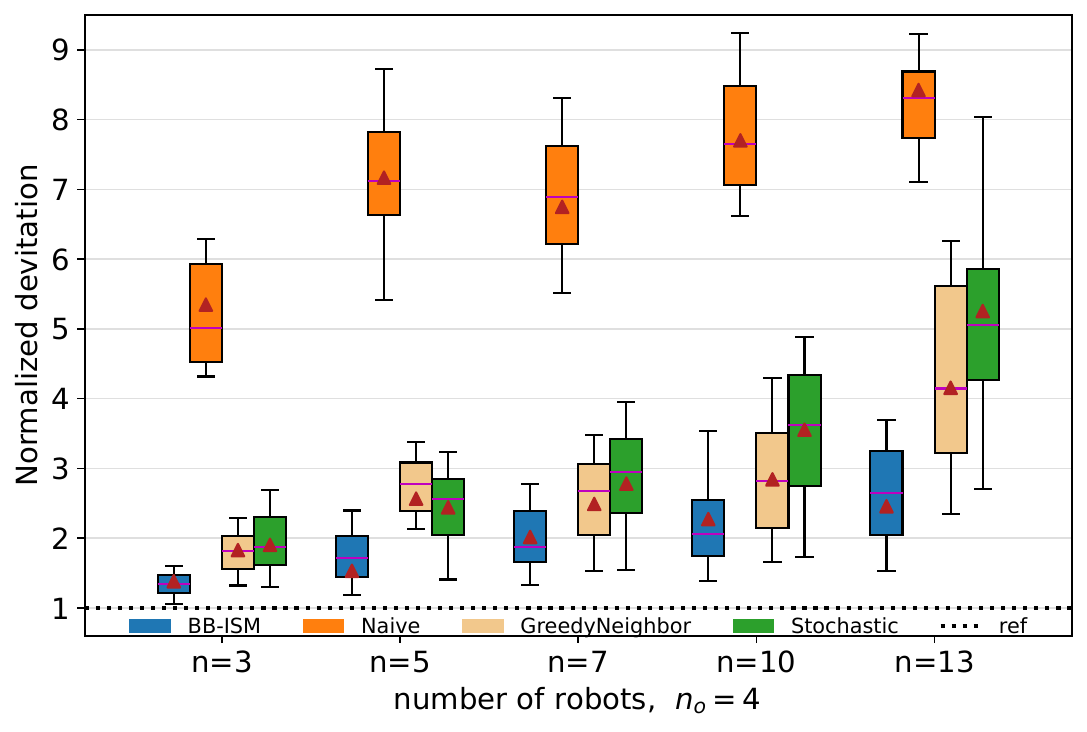}}
\caption{Baseline comparison on path deviation w.r.t. human reference paths. BB-ISM denotes our approach. There are $n_o=4$ objectives. Overall, our approach achieves much lower normalized deviation than comparisons, suggesting its effectiveness in adapting to human suggestions. 
}
\label{fig:deviation_comparison}
\end{figure}

\begin{figure}
\centerline{\includegraphics[scale=0.38]{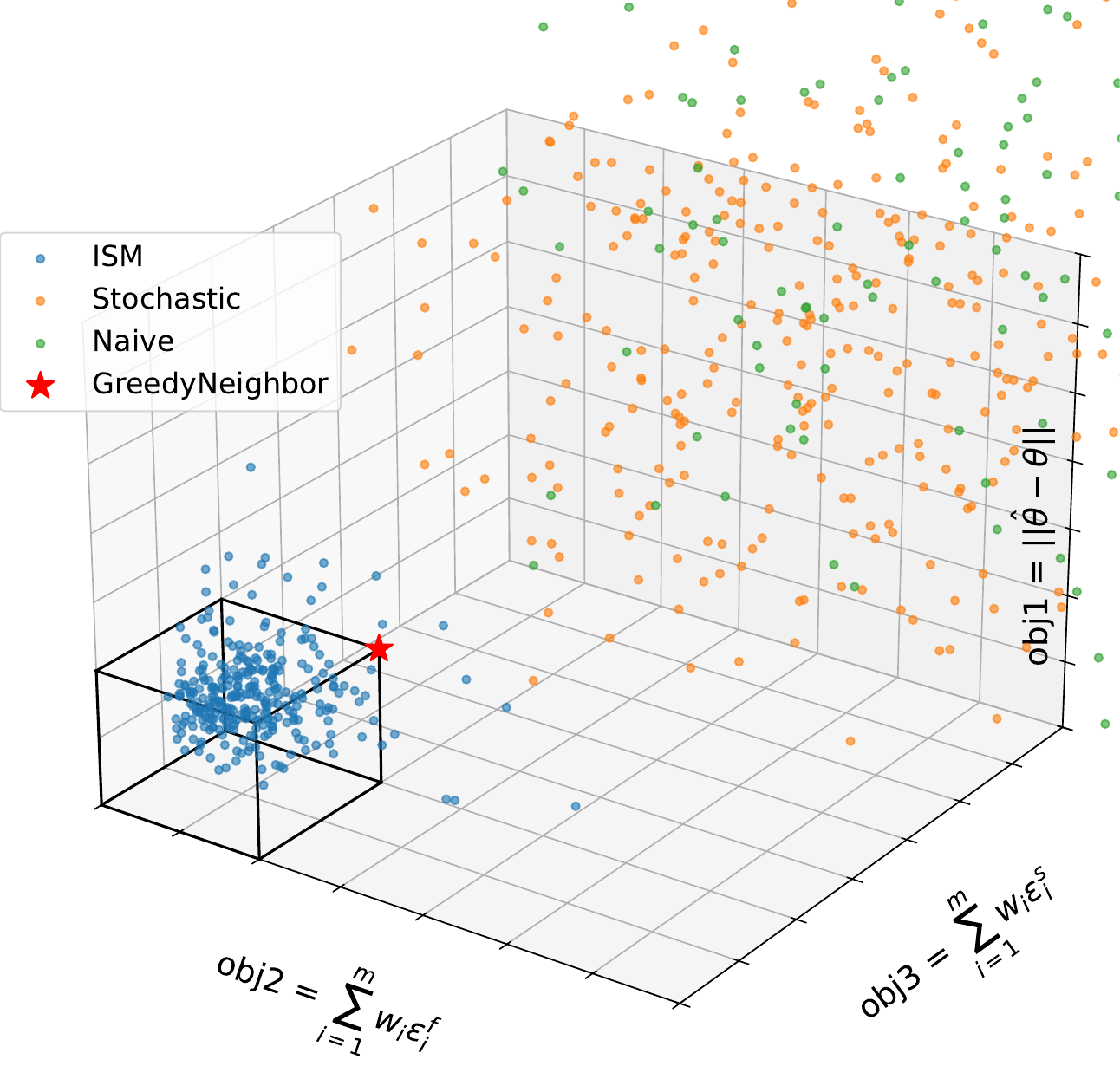}}
\caption{Baseline comparison on objective value distribution for Problem \ref{prob:multiple_suggestion}. 
}
\label{fig:multi_objective_comparison}
\end{figure}

\subsection{Algorithm Validation: BB-ISM} 
We further validate the computational performance of the proposed BB-ISM algorithm in larger problem instances using a synthetic dataset generated for the multi-objective coverage setup. 
We are interested in knowing how the optimality, running time, and memory usage change as the number of preference parameters (i.e., the dimension of $\bm{\theta}$) and the number of robots increase. 

\textbf{Optimality} We compare two types of baselines with the proposed algorithm BB-ISM in the context of coverage for event detection. The first is to directly solve Problem \ref{problem:unordered_ism} using Gurobi. We refer to this baseline as IP. The second type is to randomly sample a few suggestion orderings and solve Problem \ref{problem:ordered_ism} for each. Then, select the best solution among all the samples. We denote this baseline as RS-$z$, where $z$ refers to the number of samples. We generate test instances in the following way. First, we generate a collection of submodular objectives, each of which is associated with a particular $\bm{\theta}$. For each $\bm{\theta}$, we randomly generate some feasible suggestions and solve the ISM for each to obtain a $\hat{\bm{\theta}}$. To compare results across all $\bm{\theta}$, we use the normalized deviation $\frac{\norm{\bm{\theta}-\hat{\bm{\theta}}}}{\norm{\bm{\theta}}}$ as the comparison metric, which we aim to minimize.
%criterion to compare different approaches. The desired result should make this metric as small as possible. 
The results are shown in Fig. \ref{fig:optimality}. First, we observe that across all the test groups, i.e., different dimensions of $\bm{\theta}$, the proposed algorithm, denoted as BB, achieves the same normalized deviation, which suggests that the proposed algorithm returns the optimal solution given by IP. %as the IP does. 
For the RS baselines, they all have higher normalized deviation due to sub-optimality, which can be improved by increasing the number of samples.
%and such optimality can be improved as the number of samples increases. 

\textbf{Runtime} 
We compare the runtime of the proposed algorithm (BB) with that of integer programming using Gurobi (IP). We compare three groups for different dimensions of $\bm{\theta}$. The results are shown in Fig. \ref{fig:running_time_theta}. Since 
IP's runtime becomes computationally infeasible 
%the running time of IP will increase too fast 
after 7 robots, we only plot the IP results for fewer than 8 robots. As shown in Fig. \ref{fig:running_time_theta3}, \ref{fig:running_time_theta6}, and \ref{fig:running_time_theta9}, the proposed algorithm (BB) is much more computationally scalable than IP w.r.t. the number of robots, and the runtime increases relatively slowly across the test cases. Besides, as the dimension of $\bm{\theta}$ increases, the runtime increase is not very significant compared to that of IP.

\textbf{Peak Memory Usage} The peak memory usage results are displayed in Fig. \ref{fig:memory_theta9}. We can observe that the BB approach requires much less memory compared to IP as the number of robots increases.   

\begin{figure}
\centerline{\includegraphics[scale=0.35]{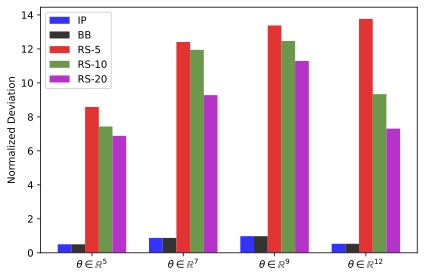}}
\caption{Optimality comparisons against IP and RS baselines. BB is ours. 
}
\label{fig:optimality}
\end{figure}

\begin{figure}[t]
    \centering
    \subfloat[Runtime $\bm{\theta} \in \mathbb{R}^3$]{
    \includegraphics[width=0.23 \textwidth]{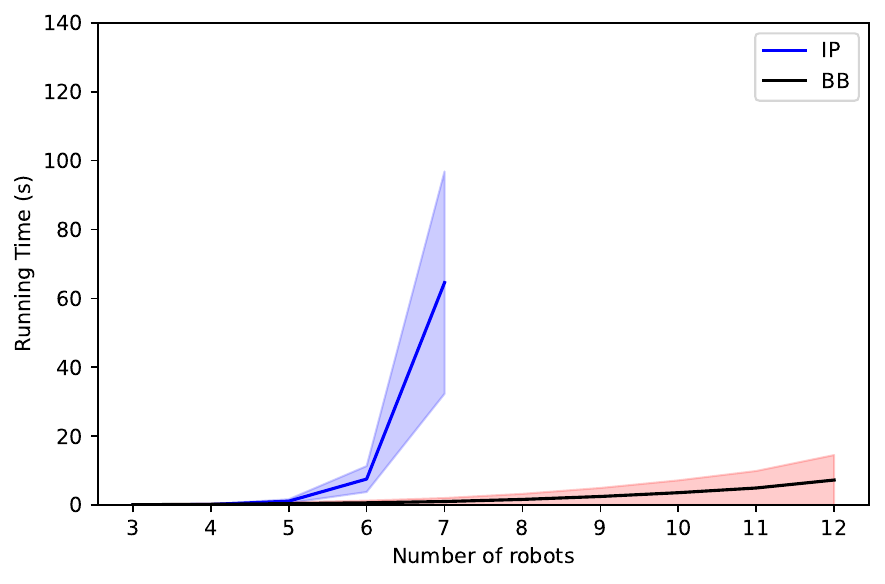}
    \label{fig:running_time_theta3}
    } 
    \subfloat[Runtime $\bm{\theta} \in \mathbb{R}^6$]{
    \includegraphics[width=0.23\textwidth]{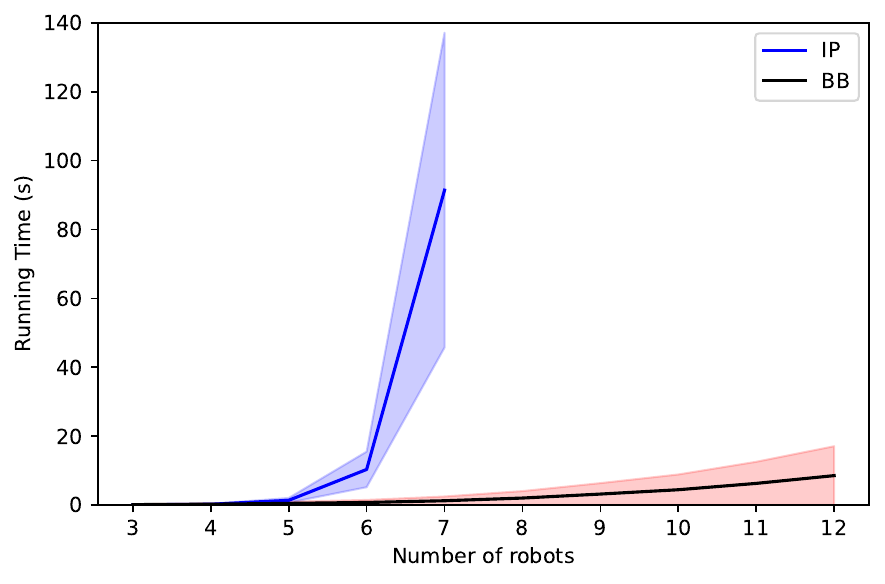}
    \label{fig:running_time_theta6}
    }\\
    \subfloat[Runtime $\bm{\theta} \in \mathbb{R}^9$]{
    \includegraphics[width=0.23\textwidth]{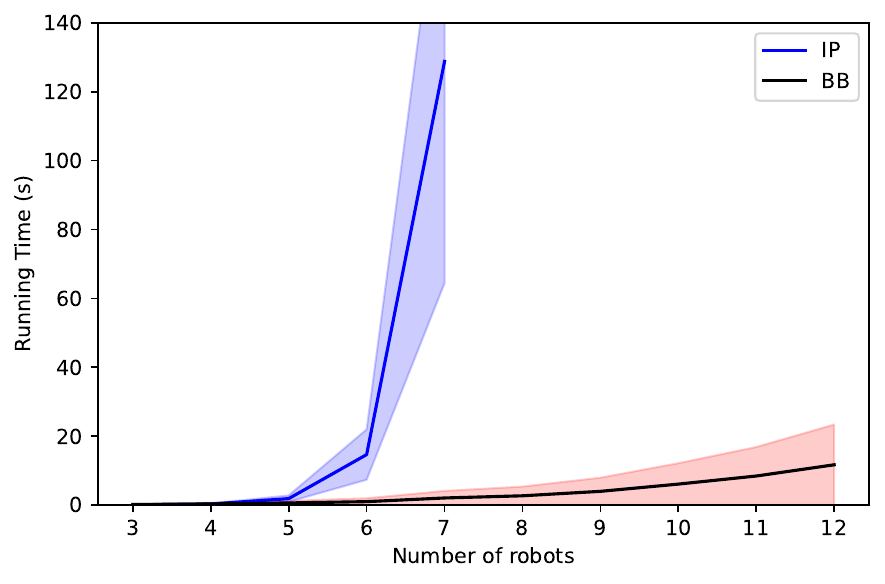}
    \label{fig:running_time_theta9}
    }
     \subfloat[Peak memory $\bm{\theta} \in \mathbb{R}^9$]{
    \includegraphics[width=0.23\textwidth]{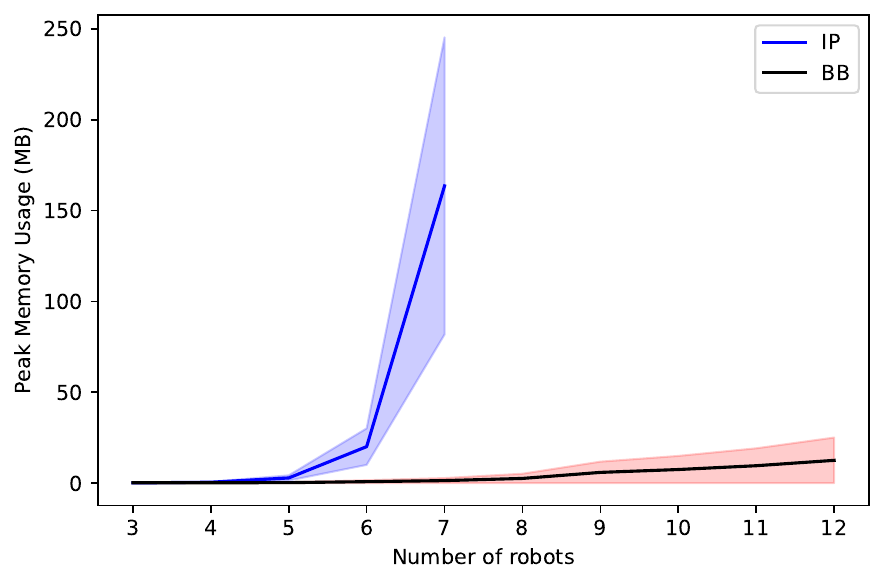}
    \label{fig:memory_theta9}
    }
    \caption{
     Running time and peak memory usage comparisons with baselines. BB is ours.}
    \label{fig:running_time_theta}
\end{figure}

\section{Conclusion}\label{sec:conclusion}
We consider a new type of inverse combinatorial optimization problem - inverse submodular optimization (ISM)- for human-in-the-loop multi-robot information gathering. With one human suggestion at each decision-making step, we show that this problem can be formulated as an MIQP and introduce a new branch and bound algorithm that can optimally solve ISM. For the case with multiple human suggestions in a given decision-making step, we formulate the problem as a multi-objective optimization and propose to solve it with Pareto MCTS. We validate the proposed algorithm in multi-robot scientific data collection and multi-objective coverage control tasks.  

This work can be viewed as an initial efforts for bridging inverse combinatorial optimization and human-in-the-loop multi-robot coordination. Many interesting aspects are worth exploring in the future. For example, when human preferences may affect both the objective and constraints, it is unclear how to define and solve ICO in such cases to accommodate human preference changes.  Another interesting direction involves efficiently solving the inverse problem when the preference-related parameters are nonlinear in the objective and constraints (e.g., in models related to human trust and proficiency) and the forward approximation algorithm is beyond greedy. %, how to efficiently solve the inverse problem?

\section*{Appendix}
\subsection{Proof of Theorem \ref{theorem:submodular_objective_scientific}}
\begin{proof}
        The objective $f(\mathcal{S}, \bm{\theta}) = \sum_{i} \theta_{i} g_{i}(\mathcal{S}) + \sum_{j} \theta_j g_{j}(\mathcal{S})$ defined in Eq. \eqref{eq:data_collection_objective} consists of two components. The first component is a weighted sum of mutual information terms. Since mutual information is a submodular set function \cite{krause2008near}, and nonnegative weighted sums preserve submodularity, this component is submodular. The second component is a weighted sum of the discrepancy term defined in Eq. \eqref{eq:objective:discrepency}. This term is modular, and weighted sums of modular functions remain modular. Finally, because the sum of a submodular function and a modular function is submodular, the overall objective function is submodular.
\end{proof}

\begin{lemma}\label{lemma:product_submodular}
    If $h_1$ and $h_2$ are both supermodular, non-decreasing (or non-increasing), and non-negative set functions defined on a finite set $\mathcal{V}$, then $h_1h_2$ is also a supermodular non-decreasing (or non-increasing) and non-negative function. 
\end{lemma}
\begin{proof}
    We prove this lemma using the definition of the supermodular function. Let us first consider the case where both functions are non-decreasing. Given two sets $\mathcal{A}$, $\mathcal{B}$ and $\mathcal{A} \subseteq \mathcal{B}$, by definition, if we add one extra element $a$ to each set, we have
    \begin{align}
        \begin{aligned}
            &h_1(\mathcal{A} \cup\{a\})\cdot h_2(\mathcal{A} \cup\{a\}) - h_1(\mathcal{A})\cdot h_2(\mathcal{A})\\
            = (&h_1(\mathcal{A} \cup\{a\})-h_1(\mathcal{A}))\cdot h_2(\mathcal{A}\cup\{a\})\\
            &\quad \quad \quad \quad + h_1(\mathcal{A})\cdot(h_2( \mathcal{A} \cup \{a\})-h_2(\mathcal{A})).
        \end{aligned}
    \end{align}
    
    By supermodularity, we have 
\begin{align}
    h_1(\mathcal{A} \cup \{a\})-h_1(\mathcal{A}) \leq h_1(\mathcal{B} \cup \{a\})-h_1(\mathcal{B}) \label{eq:lemma_supermodularity_1}\\
    h_2(\mathcal{A} \cup \{a\})-h_2(\mathcal{A}) \leq h_2(\mathcal{B} \cup \{a\})-h_2(\mathcal{B})\label{eq:lemma_supermodularity_2}.
\end{align}
By non-decreasing monotonicity, we have
\begin{align}
    h_1(\mathcal{A} \cup \{a\}) \leq h_1(\mathcal{B} \cup \{a\}) \label{eq:lemma_monotone_1}\\
    h_2(\mathcal{A} \cup \{a\}) \leq h_2(\mathcal{B} \cup \{a\}) \label{eq:lemma_monotone_2}.
\end{align}
Together with non-negativity, we have 
\begin{align}
        &\begin{aligned}\label{eq:lemma_last_step1}
            h_1(\mathcal{A} \cup \{a\})-h_1(\mathcal{A})) \cdot h_2(\mathcal{A} \cup \{a\}) + \\h_1(\mathcal{A}) \cdot (h_2(\mathcal{A} \cup \{a\})-h_2(\mathcal{A}))
        \end{aligned}
       \\ &
        \begin{aligned}\label{eq:lemma_last_step2}
             \leq  (h_1(\mathcal{B} \cup \{a\})-h_1(\mathcal{B})) \cdot h_2(\mathcal{B} \cup \{a\}) + \\h_1(\mathcal{B}) \cdot (h_2(\mathcal{B} \cup \{a\})-h_2(\mathcal{B}))
        \end{aligned}
        \\&=h_1(\mathcal{B} \cup \{a\}) \cdot h_2(\mathcal{B} \cup \{a\})-h_1(\mathcal{B} )  \cdot h_2(\mathcal{B}).
\end{align}
By definition, $h_1h_2$ is a supermodular function.

For the case where both functions are non-increasing, Eq. \eqref{eq:lemma_supermodularity_1} and Eq. \eqref{eq:lemma_supermodularity_2} still hold. The difference is that both sides of the inequalities are non-positive. In Eq. \eqref{eq:lemma_monotone_1} and Eq. \eqref{eq:lemma_monotone_2}, the inequalities should be flipped. However, since both sides of Eq. \eqref{eq:lemma_supermodularity_1} and Eq. \eqref{eq:lemma_supermodularity_2} are non-positive, we can still get the same result as those in Eq. \eqref{eq:lemma_last_step1} and Eq. \eqref{eq:lemma_last_step2}.
%\begin{align}
%    h_1(\mathcal{A} \cup \{a\}) \geq h_1(\mathcal{B} \cup \{a\}) \label{eq:lemma_monotone_de1}\\
%    h_2(\mathcal{A} \cup \{a\}) \geq h_2(\mathcal{B} \cup \{a\}) \label{eq:lemma_monotone_de2}.
%\end{align}

\end{proof}

\subsection{Proof of Theorem \ref{theorem:submodular_objective_horizon}}
\begin{theorem}[Theorem 1 in \cite{sun2017submodularity}]\label{theorem:submodular_coverage}
    If we treat the $h_j(\bm{x}(k))$ in Eq. \eqref{eq:single_step_coverage} as a set function with input set as $\bm{x}(k)$, $h_j(\bm{x}(k))$ is monotone submodular. 
\end{theorem}

It should be noted that $\bm{x}(k)$ can be viewed as a column of the selected set $\mathcal{S}$ (different path primitives correspond to different rows). As a result, $h_j(\bm{x}(t))$ is submodular w.r.t. $\mathcal{S}$. We will prove Theorem \ref{theorem:submodular_objective_horizon} using Theorem \ref{theorem:submodular_coverage} and Lemma \ref{lemma:product_submodular}. 
\begin{proof}
    By Theorem \ref{theorem:submodular_coverage}, $h_j(\bm{x}(t))$ is a monotone non-decreasing non-negative  submodular function w.r.t. $\mathcal{S}$. Therefore, $1-h_j(\bm{x}(t))$ is a monotone non-increasing non-negative  supermodular function w.r.t. $\mathcal{S}$. By Lemma \ref{lemma:product_submodular},  $\prod_{t=k}^{t={k+H}} 1-h(\bm{x}(t))$ is also a monotone non-increasing non-negative  supermodular function. Then,  $(1-\prod_{t=k}^{t={k+H}} 1-h(\bm{x}(t)))$ is a monotone non-decreasing non-negative  submodular function. Since the weighted sum of submodular functions using positive weight results in a submodular function, $f(\mathcal{S}, \bm{\theta})$ is a monotone submodular function.
\end{proof}

\section*{Acknowledgments}

This work was sponsored in part by DEVCOM Army Research Laboratory under Cooperative Agreement Number W911NF-17-2-0181. The views and conclusions contained in this document are those of the authors and should not be interpreted as representing the official policies, either expressed or implied, of the US Government. The US Government is
authorized to reproduce and distribute reprints for Government purposes notwithstanding any copyright notation herein.

\bibliographystyle{IEEEtran}
\bibliography{IEEEabrv, ICO}

\end{document}